\documentclass{ieeetj}
\pdfoutput=1
\usepackage{cite}
\usepackage{amsmath,amssymb,amsfonts}
\usepackage{graphicx,color}
\usepackage{textcomp}
\usepackage{xcolor}
\usepackage{hyperref}
\hypersetup{hidelinks}
\usepackage{algorithm}
\usepackage{algpseudocode}
\def\BibTeX{{\rm B\kern-.05em{\sc i\kern-.025em b}\kern-.08em
    T\kern-.1667em\lower.7ex\hbox{E}\kern-.125emX}}
\AtBeginDocument{\definecolor{tmlcncolor}{cmyk}{0.93,0.59,0.15,0.02}\definecolor{NavyBlue}{RGB}{0,86,125}}
\usepackage{booktabs}
\usepackage{multirow}
\usepackage{array}
\usepackage[caption=false,font=footnotesize]{subfig}
\let\labelindent\relax
\usepackage{enumitem}
\usepackage{bm}
\newtheorem{proposition}{Proposition}
\newtheorem{remark}{Remark}

\def\OJlogo{\vspace{-4pt}$<$Society logo(s) and publication title will appear here.$>$}
\def\seclogo{\vspace{10pt}$<$Society logo(s) and publication title will appear here.$>$}

\def\authorrefmark#1{\ensuremath{^{\textbf{#1}}}}
\usepackage{wrapfig}
\usepackage{xspace}
\usepackage{acronym}
\usepackage{amsmath}

\newacro{KD}{Knowledge Distillation}
\newacro{LLM}{Large Language Model}
\newacro{AI}{Artificial Intelligence}
\newacro{RLHF}{Reinforcement Learning from Human Feedback}
\newacro{API}{Application Programming Interface}
\newacro{ML}{Machine Learning}
\newacro{MLaaS}{Machine Learning as a Service}
\newacro{ADS}{Antidistillation Sampling}
\newacro{CTS}{Confidence-based Temperature Scaling}
\newacro{NTPert}{Non-Target Perturbation}
\newacro{NTPerm}{Non-Target Permutation}
\newacro{pp}{percentage point}
\newacro{DKPP}{Dark Knowledge with Permuted Predictions}
\newacro{CPS}{Cyber-Physical Systems}

\newcommand{\calD}{\mathcal{D}}

\newcommand{\calL}{\mathcal{L}}

\newcommand{\calP}{\mathcal{P}}

\newcommand{\calS}{\mathcal{S}}
\newcommand{\calT}{\mathcal{T}}

\newcommand{\calX}{\mathcal{X}}

\newcommand{\dkl}{D_{KL}}

\usepackage{autonum}

\newcommand{\bp}{\mathbf{p}}

\newcommand{\bx}{\mathbf{x}}

\newcommand{\bz}{\mathbf{z}}

\newcommand{\R}{\mathbb{R}}

\newcommand{\bdelta}{\widehat{\Delta}}
\newcommand{\pp}{\,\mathrm{pp}}
\newcommand{\lammin}{\lambda_{\min}}

\begin{document}
\receiveddate{XX Month, XXXX}
\reviseddate{XX Month, XXXX}
\accepteddate{XX Month, XXXX}
\publisheddate{XX Month, XXXX}
\currentdate{XX Month, XXXX}
\doiinfo{XXXX.2022.1234567}

\markboth{}{Ullah {et al.}}

\title{ADS-C: Antidistillation Sampling for Classification}

\author{Khawaja Abaid Ullah\authorrefmark{1}, Mohammad Javad Khojasteh\authorrefmark{1}}
\affil{Department of Electrical and Microelectronic Engineering, Rochester Institute of Technology, Rochester, NY 14623 USA}
\corresp{Corresponding author: Khawaja Abaid Ullah (email: ka4150@rit.edu).}

\begin{abstract}
Knowledge distillation enables an adversary to replicate a proprietary classifier by querying its prediction interface and training a surrogate on the returned probability vectors. Antidistillation sampling, proposed for large language models, counters this threat with an input-dependent, gradient-directed perturbation of the served distribution; its transfer to classification has not been studied. Adapting the defense to classification, we show its behavior is governed by the distribution of the teacher's per-input confidence margins. Because well-trained classifiers are severely overconfident, the direct transfer exhibits an inert window: below a closed-form-predictable threshold, it affects neither attacker nor defender; beyond it, the defense undergoes a phase transition and degrades the teacher faster than the attacker's student. Temperature softening rescales the transition in closed form, and every temperature configuration lies on the same unfavorable trade-off curve. Our method, ADS-C, composes the perturbation under a closed-form, per-input margin budget that provably preserves every served top-1 prediction, so the defended teacher's accuracy equals the undefended teacher's identically. Under this guarantee the distilled student still loses 17.4 percentage points on CIFAR-100, 29.6 on CIFAR-10, and 13.3 on Tiny-ImageNet; matching this degradation with the unmodified defense costs 27.5, 32.9, and 22.2 points of teacher accuracy. Because served labels are unchanged, a hard-label attacker gains nothing, while the defended soft output trains a student up to 29.7 points below that floor: the incentive to distill served probabilities is not merely removed but reversed. To our knowledge, ADS-C is the first antidistillation defense for classification whose utility cost is exactly zero.
\end{abstract}



\maketitle

\section{INTRODUCTION}
\label{sec:intro}
\IEEEPARstart{K}{nowledge} distillation (KD)~\cite{schmidhuber1992learning,hinton2015distilling} transfers the predictive function of a large \emph{teacher} model into a smaller \emph{student} model by training the student on the teacher's returned probability vector outputs. This very same mechanism that makes possible the compression of large models into smaller, efficient networks ~\cite{howard2017mobilenets,sanh2019distilbert,touvron2021training,gou2021knowledge}, also enables \emph{model extraction} wherein an adversary with black-box \ac{API} access systematically queries a deployed classifier, collects the returned probability vectors, and trains a student model that replicates the teacher's behavior at a fraction of the original cost~\cite{tramer2016stealing,zhao2025survey}. As recent events have highlighted, this threat is not just hypothetical. DeepSeek's 2025 release demonstrated that near-frontier capability may be obtained largely through distillation from proprietary models~\cite{cnbc2025deepseek}. Additionally, Anthropic recently disclosed an industrial-scale extraction campaign against Claude that produced over 16 million exchanges via approximately 24{,}000 fraudulent accounts~\cite{anthropic2026distillation}. Beyond direct intellectual-property loss~\cite{sun2023deep,jiang2026intellectual}, extracted models facilitate downstream membership-inference~\cite{shokri2017membership,salem2018ml}, evasion~\cite{papernot2017practical,biggio2013evasion,lowd2005adversarial}, and privacy attacks against the underlying training distribution~\cite{fredrikson2014privacy,fredrikson2015model,ateniese2015hacking}, raising the stakes especially in regulated domains such as medicine, and finance where machine learning applications are gaining traction ~\cite{haug2023artificial, hernandez2024financial}.
Beyond these settings, the implications of model extraction also extend to \ac{CPS} as KD is increasingly used in resource-constrained robotic platforms~\cite{cao2023learning}. Distillation-based extraction may also expose safety-critical control behavior in the physical world, raising security concerns for \ac{CPS} deployments.
 
The attacker consumes the returned probability vector, also known as the \textit{soft labels}, in contrast to \textit{hard labels}, where the output contains only the index of the predicted class. The inter-class structure of these soft labels, i.e., the relative probabilities of the classes other than the top class, constitutes the \textit{dark knowledge} of the given model~\cite{hinton2015distilling}. This dark knowledge makes soft-label distillation markedly more effective than training on hard labels alone. A naive defense against such distillation attempts would be to always return uninformative output, but that would make the deployed model useless for regular users. The natural question is therefore how to corrupt the dark knowledge a model deployed through an API serves without destroying the utility of the given model for legitimate users. Every inference-time defense implicitly negotiates this trade-off, and we argue the trade-off rate should be measured against an explicit break-even: a defense that removes one \ac{pp} of attacker accuracy per \ac{pp} of teacher accuracy, a 1:1 trade-off, is no better than deploying a worse model. A useful defense must beat the 1:1 trade-off, ideally by a wide margin in favor of the model being defended.

\ac{ADS}~\cite{savani2025antidistillation} is a compelling starting point. Proposed for autoregressive language models, \ac{ADS} perturbs each next-token distribution with an \emph{input-dependent, gradient-directed} penalty, computed from a hidden proxy student, that estimates how much serving each token would help a distilling student. Unlike static perturbation defenses, which apply a fixed transformation that an attacker can invert or meta-learn~\cite{lee2019defending,chen2023d}, the ADS perturbation is recomputed per query from hidden state. Whether this idea transfers to classification, which is among the most broadly exposed extraction surfaces in practice, has remained open. The question is not merely one of engineering: in a classification API the distilling student ingests the \emph{entire} served vector through its KL objective, whereas LLM sampling collapses the perturbation onto a single emitted token, so per unit of perturbation \ac{ADS} should be \emph{more} potent in classification, not less.

This paper adapts ADS to classification, analyzes the mechanism that governs its behavior there, and repairs its ineffectiveness with a composition rule whose utility cost is provably zero. Our contributions are as follows.
\begin{enumerate}
\item \textbf{Adaptation.} We instantiate the ADS penalty in the classification setting, preserving its derivation, and identify the structural differences from the LLM setting (Sec.~\ref{sec:ads}).
\item \textbf{Diagnosis.} We show that the behavior of the transferred defense is determined by the distribution of per-input \emph{flip thresholds}, whose cumulative distribution, before any attack or defense is run, accurately predicts the defended teacher's accuracy cost (Sec.~\ref{sec:diagnosis}). Classifier overconfidence~\cite{guo2017calibration} renders the defense inert at low strengths and unfavorable at high strengths; the perturbation itself is faithful and correctly aimed at dark knowledge, but \emph{margin-blind}.
\item \textbf{Analysis of temperature softening.} We show that softening the served distribution rescales the phase transition in closed form, and confirm the prediction in trained students: softening amplifies the perturbation at low strengths while collapsing the teacher at moderate ones, and no temperature configuration improves the trade-off (Sec.~\ref{sec:softening}).
\item \textbf{ADS-C.} We compose the ADS perturbation under a closed-form, per-input, margin-aware budget: each input receives the largest penalty strength that provably keeps its served top-1 prediction unchanged (Proposition~\ref{prop:argmax}), so the defended teacher's accuracy equals the undefended teacher's identically. The surviving perturbation still degrades the distilled student substantially, by $17.4\pp$ on CIFAR-100, $29.6\pp$ on CIFAR-10, and $13.3\pp$ on Tiny-ImageNet at zero utility cost (Secs.~\ref{sec:adsc}--\ref{sec:results}). Against the complementary hard-label attacker the defended soft output proves strictly worse as a distillation signal than the argmax it accompanies. The soft-label student falls $17.1\pp$ and $29.7\pp$ below what the hard-label attacker achieves through distillation (Sec.~\ref{sec:results_controls}).
\item \textbf{Relaxations of the guarantee.} We introduce two mechanisms with closed-form-predictable utility cost: first, stochastic enforcement of the budget; second, a minimum served strength. Both trade a controlled amount of the guarantee for attacker uncertainty. We price each relaxation against the attacker's best response between soft- and hard-label extraction, and report a negative result on margin-ranked sacrifice that further clarifies the damage mechanism (Sec.~\ref{sec:dials}).
\item \textbf{Evaluation methodology and reproducibility.} We argue that once utility cost is zero, defenses should be compared at matched served fidelity rather than matched penalty strength, and show that the two comparisons can disagree (Sec.~\ref{sec:results_frontier}). We release a library and every experiment driver: \href{https://github.com/the-kas-lab/ADS-C}{https://github.com/the-kas-lab/ADS-C}.
\end{enumerate}

\section{RELATED WORK}
\label{sec:related}
\textbf{Distillation and dark knowledge.} Distillation trains a student on a teacher's output distributions~\cite{schmidhuber1992learning,bucilua2006model,ba2014deep,hinton2015distilling}; its advantage over hard-label training is attributed to the inter-class structure of those distributions~\cite{hinton2015distilling,furlanello2018born,tang2020understanding}. Furlanello \emph{et al.}~\cite{furlanello2018born} showed diagnostically that destroying non-target structure destroys most of distillation's benefit, providing evidence that dark knowledge is the asset a defense should target.

\textbf{Model extraction.} Extraction attacks replicate a deployed model from query access~\cite{tramer2016stealing,orekondy2019knockoff,truong2021data,jahan2025black}; surveys~\cite{oliynyk2023know,zhao2025survey,jiang2023comprehensive,xu2024survey} document their practicality. We consider the distillation-based extraction attacker, who trains on returned soft labels which is the strongest and most common variant when probability vectors are exposed.

\textbf{Extraction defenses.} Training-time defenses harden the teacher's weights~\cite{ma2021undistillable,yilmaz2025adversarial,liang2024defending,cheng2025misleader} but require retraining and access to the original corpus, which is often undesirable or even infeasible at deployment scale. Detection and monitoring~\cite{juuti2019prada,gurve2024misguide,chakraborty2025radep} flag anomalous query streams, but rely on out-of-distribution assumptions that prove futile against attackers with in-distribution data~\cite{cheng2025misleader} and are circumvented by Sybil accounts~\cite{douceur2002sybil,yu2020cloudleak}. Watermarking~\cite{uchida2017embedding,jia2021entangled,szyller2021dawn,wang2024defense} supports post-hoc attribution but erects no barrier against the extraction itself. \emph{Perturbation} defenses are closest to our approach. These defenses modify the served probabilities: Reverse Sigmoid~\cite{lee2019defending}, prediction poisoning~\cite{orekondy2019prediction}, ModelGuard~\cite{tang2024modelguard}, noise-transition learning~\cite{wu2024efficient}, and information-theoretic output compression~\cite{fang2026towards}. All of these pay for protection with utility, and the static members of the family being fixed functions of the input are invertible by meta-learned recovery attacks such as D-DAE~\cite{chen2023d}. None offers a per-query \emph{guarantee} on the served top-1 prediction.

\textbf{ADS.} Antidistillation Sampling~\cite{savani2025antidistillation} is input-dependent and gradient-directed wherein a hidden proxy student defines, per query, a direction in output space that maximally harms a distilling student. It was derived and evaluated for LLM reasoning traces. Two questions are open: whether it transfers to classification, and what it costs the defender there. We answer both: the transfer's behavior is gated entirely by classifier overconfidence, and the repair (a margin budget at composition) makes the cost exactly zero.

\section{THREAT MODEL AND PROBLEM SETUP}
\label{sec:threat}
Figure~\ref{fig:threat_model} summarizes the setting. A defender trains a classifier $\calT\colon\calX\to\Delta^{K-1}$ mapping an input space $\calX$ to the probability simplex over $K$ classes and deploys it behind a query API that returns, for each query $\bx\in\calX$, a served probability vector $\hat{\bp}(\bx)\in\Delta^{K-1}$.

\textbf{Adversary.} The adversary holds (i) black-box access to the API with an arbitrary but bounded query budget; (ii) an unlabeled query corpus drawn from the same distribution $\calD$ over $\calX$ as the teacher's training data---we make no out-of-distribution assumption, which disables detection-style defenses by construction; and (iii) a student architecture $\calS$ of its choosing. The adversary lacks the teacher's parameters, and ground-truth labels. With no labels, the attacker's objective reduces from the standard hard-plus-soft KD loss to pure soft-label matching,
\begin{equation}
\label{eq:ukd}
\calL_{\mathrm{UKD}} \;=\; \dkl\bigl(\hat{\bp}(\bx)\,\big\|\,\bp_\calS(\bx)\bigr),
\end{equation}
evaluated at temperature $1$: the served vector is the attacker's sole supervisory signal, and hence the natural locus of defense.

\textbf{Why the soft-label attacker.} We scope the threat to the \emph{dark-knowledge distiller}, and this scoping is principled rather than convenient. The attacker's advantage over training on public hard labels is the served soft-label structure; an attacker who reads only the served argmax gains nothing that a labeling function of equal top-1 accuracy would not provide. Our defense will, by construction, never alter the served argmax (Sec.~\ref{sec:adsc}), so hard-label distillation against the defended API yields exactly what it would yield against the undefended API (verified empirically in Sec.~\ref{sec:results_controls}): the defense removes the soft-label advantage. The exploitation of the guarantee's determinism itself is treated in Sec.~\ref{sec:dials}.

\textbf{Defender.} The defender owns the teacher, a modest labeled holdout set, and a compact \emph{proxy} student $\calP$ that stands in for the unknown attacker student; it cannot retrain the teacher and must act at inference time only. The defender's goals are: (a) \emph{utility}: legitimate users, who consume the top-1 prediction, must be unharmed; (b) \emph{distillation resistance}: a student trained via \eqref{eq:ukd} on the served outputs should be substantially degraded.

\textbf{Metrics.} (A summary of notation is provided in Appendix~\ref{app:notation}.) We report three quantities. \emph{Teacher top-1}: accuracy of the served argmax. \emph{Student top-1}: accuracy of the distilled student. \emph{Served fidelity}: $\dkl(\bp_\calT \,\|\, \hat{\bp})$ averaged over the evaluation set, the total distributional perturbation the defense spends; this axis quantifies the defense's ``corruption budget'' and becomes essential once utility cost is zero (Sec.~\ref{sec:results_frontier}).

\begin{figure}[htbp]
    \centering
    \subfloat[The adversary constructs a labeled dataset by querying the victim API with unlabeled samples.\label{fig:threat_model_part_1}]{%
        \includegraphics[width=0.45\textwidth]{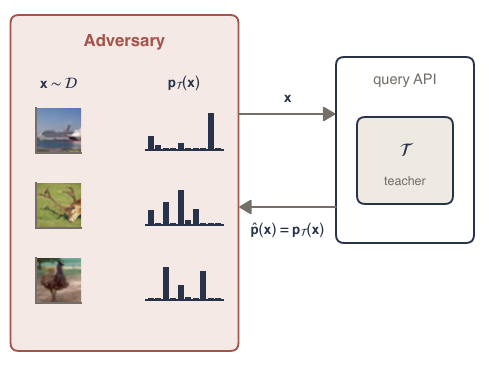}%
    }
    \hfill
    \subfloat[The adversary trains a surrogate student on the returned soft labels.\label{fig:threat_model_part_2}]{%
        \includegraphics[width=0.45\textwidth]{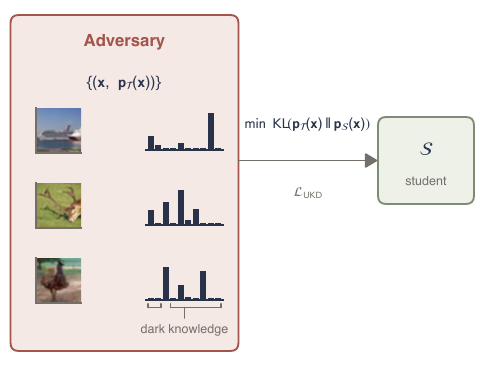}%
    }
    \caption{The unauthorized knowledge-distillation threat model.}
    \label{fig:threat_model}
\end{figure}

\section{ANTIDISTILLATION SAMPLING FOR CLASSIFICATION}
\label{sec:ads}
\subsection{Preliminaries}
\label{sec:prelims}
A pretrained teacher with parameters $\theta_\calT$ produces logits $\bz_\calT=\calT(\bx;\theta_\calT)\in\R^{K}$; the softmax $\sigma(\bz)_i = \exp(z_i)/\sum_j \exp(z_j)$ yields probabilities $\bp_\calT$. Only logit differences matter: adding a constant to $\bz$ leaves $\sigma(\bz)$ unchanged, so the margin $z_t - z_j$ between the top class $t$ and a rival $j$ is the invariant quantity that decides the argmax. Distillation trains a student to minimize \eqref{eq:ukd} against the served distribution.

\subsection{The ADS penalty}
\label{sec:penalty}
ADS asks, per query and per class: \emph{if the student came to believe class $y$ is slightly more likely on this input, would that make the student worse?} Formally, let $\ell(\theta_\calP)$ be a downstream loss of the proxy on the defender's labeled holdout. We take $\ell(\theta_\calP) = \frac{1}{N}\sum_{i=1}^{N}\dkl\bigl(\bp_\calT(\bx_i)\,\|\,\bp_\calP(\bx_i)\bigr)$,
the proxy's mean disagreement with the teacher over the holdout, following~\cite{savani2025antidistillation},
and let $g = \nabla\ell(\theta_\calP)$ be its gradient. One SGD step of distillation on class $y$ moves $\theta_\calP$ along $\nabla_{\theta}\log p(y\mid\bx;\theta_\calP)$; the resulting change in $\ell$ is, to first order, the inner product of that update with $\nabla\ell$. Exactly as in the LLM derivation, symmetry of the directional derivative converts this into a finite difference over two \emph{frozen clones} of the proxy, $\theta_{\calP\pm} = \theta_\calP \pm \varepsilon g$:
\begin{equation}
\label{eq:fd}
\widehat{\Delta}(y\mid\bx) \;=\; \frac{\log p(y\mid\bx;\theta_{\calP+}) - \log p(y\mid\bx;\theta_{\calP-})}{2\varepsilon}.
\end{equation}
A positive $\widehat{\Delta}(y\mid\bx)$ means that serving mass on class $y$ raises the proxy's holdout loss, which is good for the defender. The defended distribution adds the penalty to the teacher's log-probabilities:
\begin{equation}
\label{eq:ads_compose}
\hat{\bp}(\cdot\mid\bx) \;=\; \sigma\bigl(\log\bp_\calT(\cdot\mid\bx) + \lambda\,\bdelta(\cdot\mid\bx)\bigr),
\end{equation}
with $\lambda \ge 0$ the defense strength; $\lambda=0$ recovers the undefended API.

\textbf{Parameter choices.} (i)~Defining $\ell$ as a per-sample mean fixes the scale of $g$ independently of the holdout size; empirically $\lVert g\rVert$ is of order one ($1.00$ on CIFAR-100, $0.77$ on CIFAR-10, $1.37$ on Tiny-ImageNet), so $\lambda$ operates on a comparable scale across datasets. We sweep $\lambda\in[0,3]$, extended to $\lambda=10$ for unmodified ADS to map the saturation regime (Sec.~\ref{sec:transition}). (ii)~The step $\varepsilon$ is calibrated, not tuned: we select $\varepsilon$ to minimize the discrepancy between \eqref{eq:fd} and the exact Jacobian--vector product it approximates, obtaining $\varepsilon=3\times10^{-3}$ on both CIFAR datasets and $\varepsilon=10^{-2}$ on Tiny-ImageNet, each with a magnitude ratio of $\approx\!1.00$ (Appendix~\ref{app:epsilon}); the finite difference is a faithful estimate, which will matter for the diagnosis of Sec.~\ref{sec:diagnosis}. (iii)~The proxy (ResNet-20 on the CIFAR datasets, ResNet-18 on Tiny-ImageNet) is substantially smaller than the teacher and is trained on the teacher's training split, with a disjoint $10\%$ held out for computing $g$; the two clones are built once, offline, and frozen, so a defended query costs two small forward passes and no gradient computation. (iv)~The margin floor $m$ introduced in Sec.~\ref{sec:adsc} is analyzed in Appendix~\ref{app:margin_floor}.

\subsection{ADS Case for Classification}
\label{sec:llm_vs_cls}
The penalty~\eqref{eq:fd} is structurally identical to the next-token formulation of~\cite{savani2025antidistillation}; what differs is downstream consumption. In autoregressive generation the perturbation tilts the next-token distribution,
but the attacker observes only a single sampled token per step: a one-hot
realization whose expectation carries the tilted distribution, so the signal on
the non-sampled coordinates of $\bdelta$ reaches the student only in expectation,
across many draws. In classification distillation, the attacker observes the served
vector itself: the KL objective~\eqref{eq:ukd} ingests every coordinate of
$\bdelta$ exactly, on every query. The same perturbation is thus delivered through
a noiseless channel rather than estimated through a stochastic one, and per unit
of perturbation ADS should be \emph{more} potent in classification, not less.
The next section shows that, despite this, the direct transfer is inert over a wide regime, then trades at an unfavorable rate, and finally saturates with both parties near their accuracy floors. In addition, it locates the cause of this behavior not in the penalty but in its composition with the teacher.

\section{DIAGNOSIS OF ADS INEFFECTIVENESS}
\label{sec:diagnosis}
Swept over $\lambda$, the unmodified ADS exhibits three behaviors in sequence that make it undesirable for deployment in the classification domain. For $\lambda \lesssim 0.2$ it is \emph{inert}: on both CIFAR datasets, student and teacher accuracy each change by at most $1$--$2\pp$, within seed noise. The defense measurably affects neither party. Beyond that, it becomes active at an unfavorable rate: on CIFAR-100 at $\lambda=0.5$, the distilled student drops $5.9\pp$ while the teacher drops $12.2\pp$, a trade-off rate of $0.48$, half of break-even; on CIFAR-10 the rate is $0.85$ ($9.6\pp$ student for $11.3\pp$ teacher). Finally, for $\lambda \gtrsim 2$ the sweep \emph{saturates}: the trade-off remains sub-linear while both models degrade toward their floors, and by $\lambda=3$ the defended teacher's accuracy has fallen to the level of the very student it defends against (Fig.~\ref{fig:main_tradeoff}, diagonal curves). This section explains all three behaviors through a single mechanism and derives the design requirement for the remedy.

\subsection{Overconfidence of Classifiers}
\label{sec:overconf}
Modern, well-trained networks are systematically overconfident~\cite{guo2017calibration,pereyra2017regularizing}: their top-class probability $p_{\max}$ concentrates near 1. Table~\ref{tab:overconfidence} quantifies this for our teachers. On CIFAR-10 the median served $p_{\max}$ exceeds $0.99999$. The median non-target mass, the entire surface on which dark knowledge lives and on which an additive penalty can visibly act, is below $10^{-6}$. CIFAR-100 is less extreme (median non-target mass $\approx 6\times10^{-3}$), and Tiny-ImageNet sits between the two ($4\times10^{-4}$) despite having twice CIFAR-100's classes. The comparison is deliberate: the Tiny-ImageNet teacher is a strong ImageNet-pretrained ResNet-50 fine-tune, whereas the CIFAR teachers are trained from scratch, and the pretrained model's strength overrides the intuition that more classes produce more diffuse outputs. We find that overconfidence is determined by model quality, not class count, and strong pretrained models are precisely what production APIs deploy, so the overconfident regime studied here is the realistic one rather than a small-benchmark artifact. (At the extreme of output dimension, LLMs over $>$50k-token vocabularies operate in a naturally diffuse regime which is one reason the original ADS formulation could presuppose sufficient non-target mass.)

The overconfidence is also \emph{heterogeneous}, and this is the central observation. The top-logit margin $z_t - z_{(2)}$ (top minus runner-up) spans two orders of magnitude within a single dataset: on CIFAR-100 its 10th/50th/90th percentiles are $0.7$/$5.5$/$15.0$ log-units; on Tiny-ImageNet, $1.1$/$8.3$/$17.0$; on CIFAR-10, $3.9$/$14.4$/$21.0$. Most inputs have margins far too large for a moderate output perturbation to overcome, while a persistent minority --- $13.4\%$, $9.4\%$, and $2.4\%$ of inputs, respectively, have margins below one log-unit---lie close to a prediction flip. Any defense that treats all inputs identically will therefore be simultaneously too weak on the head of this distribution and too strong on its tail. To our knowledge, no prior antidistillation work has examined how this confidence structure interacts with an output-perturbing defense; the remainder of this section shows the interaction is decisive.

\begin{table}[!t]
\centering
\renewcommand{\arraystretch}{1.05}
\caption{Teacher confidence structure (median over the evaluation set).}
\label{tab:overconfidence}
\begin{tabular}{lccc}
\toprule
\textbf{Benchmark} & \textbf{$K$} & \textbf{$p_{\max}$} & \textbf{Non-target mass}\\
\midrule
CIFAR-10 (RN-56)   & $10$  & $>0.99999$ & $8\times10^{-7}$\\
Tiny-ImageNet (RN-50$^\dagger$) & $200$ & $0.9996$ & $4\times10^{-4}$\\
CIFAR-100 (RN-110) & $100$ & $0.994$    & $6\times10^{-3}$\\
\bottomrule
\multicolumn{4}{l}{\footnotesize $^\dagger$ImageNet-pretrained, fine-tuned at $64\times64$.}
\end{tabular}
\end{table}

\begin{figure}[!t]
  \centering
  \includegraphics[width=\columnwidth]{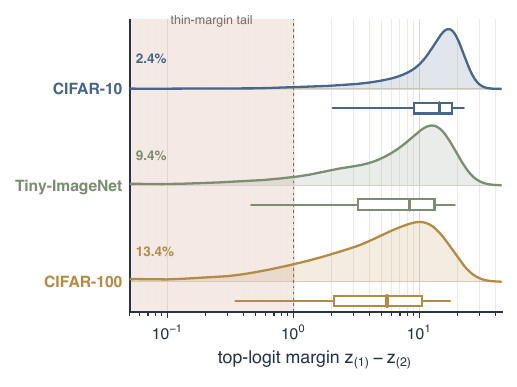}
  \caption{Overconfidence and its heterogeneity. Per-dataset distribution of the top-logit margin $z_{(1)}-z_{(2)}$ (top class minus runner-up), drawn as a density (above) and a box (below) that captures the median, interquartile range, and p5--p95 whisker on a logarithmic axis as within every dataset the margin varies over orders of magnitude.}
  \label{fig:overconfidence}
\end{figure}

\subsection{Soundness of ADS Penalty}
\label{sec:poison_anatomy}
Before locating the futility at composition we rule out the penalty. Two verifications (Fig.~\ref{fig:poison_anatomy}):

\emph{Faithfulness.} At the calibrated $\varepsilon$, the finite-difference $\bdelta$ matches the exact directional derivative with magnitude ratio $\approx\!1.00$ on both datasets (Appendix~\ref{app:epsilon}). The poison is not a numerical artifact.

\emph{Aim.} $\bdelta$ concentrates exactly where dark knowledge lives. Averaged over inputs, its value at the teacher's top class is near zero (mean $-1.1$ on CIFAR-100, $-0.4$ on CIFAR-10) while deep non-target ranks receive large positive pushes (up to $+5.6$ and $+12.3$ respectively); the class each input's poison boosts hardest sits at teacher rank $\ge 5$ for $98\%$ of CIFAR-100 and $63\%$ of CIFAR-10 inputs. ADS corrupts the inter-class structure beneath the prediction.

\begin{figure}[!t]
  \centering
  \includegraphics[width=\columnwidth]{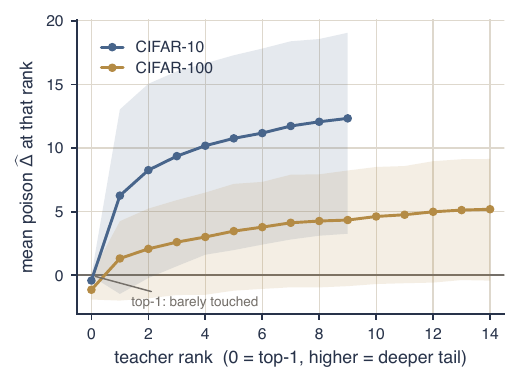}
  \caption{Anatomy of the ADS penalty. Mean $\widehat{\Delta}$ as a function of the teacher's class rank; near zero at rank 0, positive across the non-target tail. The penalty is aimed at dark knowledge, not at the argmax.}
  \label{fig:poison_anatomy}
\end{figure}

\subsection{Margin Blindness}
\label{sec:margin_blind}
Whether the served argmax survives \eqref{eq:ads_compose} is a per-input race between two numbers. For each rival $j\neq t$, let $c_j = \widehat{\Delta}_j - \widehat{\Delta}_t$ be the rate (per unit $\lambda$) at which the poison closes the margin to $j$; we define the poison's \emph{contrast} on input $\bx$ as the largest such rate,
\begin{equation}
\label{eq:contrast}
c(\bx) \;=\; \max_{j\neq t}\; c_j,
\end{equation}
the rate at which the strongest rival gains on the top class $t$. The served top-1 flips exactly when $\lambda$ exceeds the \emph{flip threshold}
\begin{equation}
\label{eq:flip}
\lambda^\ast(\bx) \;=\; \min_{j\neq t:\,c_j>0}\;\frac{z_t - z_j}{c_j},
\end{equation}
with $\lambda^\ast(\bx)=+\infty$ when no rival satisfies $c_j>0$.
Two measurements now explain the unfavorable exchange. First, the contrast is large: its median is $17.4$ (CIFAR-100) and $16.3$ (CIFAR-10) log-units per unit $\lambda$---comparable to the 90th-percentile margin and $25\times$ the 10th-percentile margin. Second, and critically, the contrast is \emph{margin-blind}: its Spearman correlation with the margin is $\rho = +0.12$ (CIFAR-100) and $+0.18$ (CIFAR-10). The poison's per-input strength is unrelated to the margin it must respect. Nothing in the objective $\ell$ ever sees the teacher's margins, so nothing in $\bdelta$ respects them.

The consequence follows directly from \eqref{eq:flip} and the thin tail of Fig.~\ref{fig:overconfidence}: as $\lambda$ grows, served predictions flip in order of their margins, beginning with the thin-margin inputs. At $\lambda=0.5$, a global penalty has already flipped $26.1\%$ of CIFAR-100 and $14.8\%$ of CIFAR-10 served predictions (Fig.~\ref{fig:composition_geometry}). These flips constitute the entirety of the teacher's cost, and they are an inefficient attack, for a reason the two parties' metrics make precise. The teacher's cost is discontinuous in the perturbation: the moment the top two entries cross, a flip registers as a full served error. The student's supervision is the served vector, which is continuous: a flip just past its threshold crosses at near-parity, so the training signal on that row is almost unchanged. Moreover, low-$\lambda$ flips land exclusively on thin-margin inputs, whose near-uniform outputs carry little of the inter-class structure a distilling student exploits (Sec.~\ref{sec:dials_lmin} confirms this independently: corruption concentrated on thin-margin rows barely moves the student). The defender thus pays full price for perturbations the student scarcely notices---at $\lambda=0.5$ on CIFAR-100 the flip component costs the teacher $12.2\pp$ while accounting for only $3.1\pp$ of student damage---whereas the argmax-safe corruption underneath is what a distilling student cannot discount; we quantify this decomposition in Sec.~\ref{sec:results_decomp}.

\begin{figure}[!t]
  \centering
  \subfloat[Margin-blindness]{\includegraphics[width=\columnwidth]{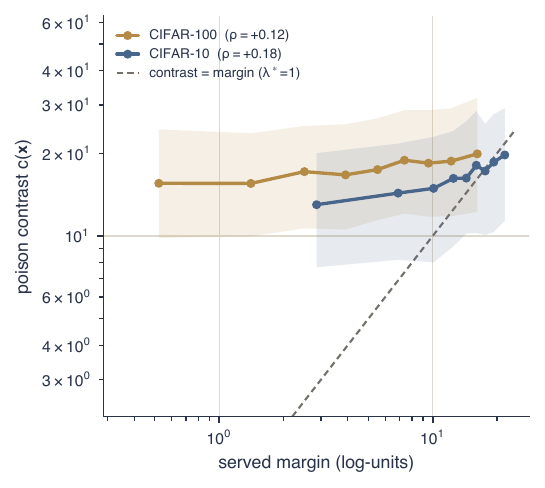}}\\
  \subfloat[Flip cascade]{\includegraphics[width=\columnwidth]{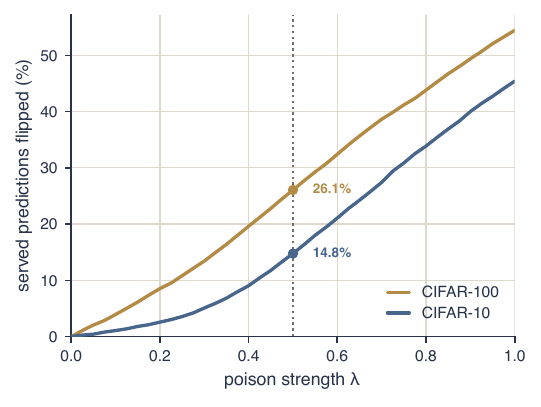}}
  \caption{Raw ADS penalty on both teachers. (a)~Median poison contrast $c(\bx)$ against served margin, binned (log--log; band: middle $50\%$), with the $c(\bx)=$~margin boundary at which a served prediction flips at $\lambda=1$. Each dataset's contrast is \emph{flat} in the margin (Spearman $\rho=+0.12$ and $+0.18$) and sits far \emph{above} the boundary across almost the whole margin range. The push is large (median $17.4$/$16.3$ log-units) and unrelated to the margin it must respect. It meets the boundary only at the largest margins. (b)~The consequence of \eqref{eq:flip}: as $\lambda$ grows, served predictions flip in order of their margins, beginning with the thin-margin tail; by $\lambda=0.5$ a global penalty has already flipped $26.1\%$ (CIFAR-100) and $14.8\%$ (CIFAR-10) of served predictions.}
  \label{fig:composition_geometry}
\end{figure}

\subsection{The inert window and the phase transition}
\label{sec:transition}
In addition to explaining the flips, Equations \eqref{eq:contrast}--\eqref{eq:flip}  make the entire teacher-cost curve predictable in closed form, before any student is trained. Let
\begin{equation}
\label{eq:flipcdf}
F(\lambda) \;=\; \Pr\bigl[\lambda^\ast(\bx) \le \lambda\bigr]
\end{equation}
be the cumulative distribution of the flip threshold over the evaluation set, computable from the teacher's and the proxy's outputs in a single forward pass. The defended teacher's accuracy drop under unmodified ADS is then the fraction of \emph{correctly classified} inputs whose threshold has been crossed. This prediction tracks the measured, trained teacher drop to within ${\approx}1$--$2\pp$ at every $\lambda$ on both datasets (CIFAR-100 at $\lambda=1$: predicted $34.9$ vs.\ measured $33.3\pp$; CIFAR-10 at $\lambda=2$: $69.8$ vs.\ $69.0\pp$). The teacher cost of unmodified ADS is thus fully determined by the flip-threshold distribution.

$F$ inherits its shape from the margin heterogeneity of Sec.~\ref{sec:overconf}, and that shape organizes the sweep into three regimes (Fig.~\ref{fig:transition}):

\begin{itemize}[leftmargin=1.2em]
\item \textbf{Inert window} ($\lambda$ below the 5th-percentile flip threshold: $\lambda<0.12$ on CIFAR-100, $\lambda<0.30$ on CIFAR-10). $F\approx0$, and---because the same overconfidence leaves almost no non-target mass for the poison to reweight---the \emph{student} is unaffected as well: measured student and teacher effects are both within $1$--$2\pp$ of the undefended anchors, indistinguishable from seed noise. In this regime the defense has no measurable effect on either party.
\item \textbf{Phase transition} ($0.35\lesssim\lambda\lesssim1$). $F$ rises through its steep section; teacher cost and student damage both grow from noise level to tens of pp within about $1.5$ octaves of $\lambda$---a narrow band of a control parameter that we sweep over four orders of magnitude. We define the \emph{transition threshold} $\lambda_c$ as the largest $\lambda$ with $F(\lambda)\le1\%$ ($\lambda_c\approx0.02$ on CIFAR-100, $0.09$ on CIFAR-10 by this strict definition; the teacher-cost onset---drop exceeding $1\pp$---lies at $0.066$ and $0.135$, respectively). $\lambda^\ast$ is approximately log-normal, so $F$ is a smooth sigmoid in $\log\lambda$ whose inflection lies at the median flip threshold ($0.91$/$1.09$).
\item \textbf{Saturation} ($\lambda\gtrsim2$). Both parties approach their floors and the sweep exhausts itself. Per-unit-$\lambda$ student damage falls by two orders of magnitude ($20.3\pp$ near $\lambda=1$ to $0.2\pp$ at $\lambda=10$ on CIFAR-100; $33.6$ to $0.3\pp$ on CIFAR-10) as the student completes its fall to---and slightly through---the chance floor ($0.6\%$ vs.\ $1\%$ chance; $6.3\%$ vs.\ $10\%$): deeply poisoned vectors are systematically wrong, not merely uninformative. The teacher's tail is equally exhausted ($64.0\to70.2\pp$ and $77.6\to86.7\pp$ of drop over $\lambda=3\to10$), and \eqref{eq:flipcdf} predicts every measured tail point to within $0.6\pp$. By $\lambda=3$ the defended teacher's accuracy has fallen to the level of the very student it defends against ($7.2\%$ vs.\ $4.2\%$; $15.1\%$ vs.\ $12.9\%$), and by $\lambda=10$ on CIFAR-10 \emph{below} it ($6.0\%$ served vs.\ $6.3\%$ distilled): the direct transfer does not merely trade unfavorably in this regime---it eliminates, and finally inverts, the accuracy advantage it exists to protect.
\end{itemize}

Note the ordering: CIFAR-10, the \emph{more} overconfident teacher, has the \emph{wider} inert window showing that inertness tracks overconfidence, as the mechanism predicts. The same computation places Tiny-ImageNet's transition threshold at $\lambda_c\approx0.06$, between the two CIFAR values, consistent with its intermediate overconfidence (Table~\ref{tab:overconfidence}); and its closed-form $F(\lambda)$ predicts the measured teacher cost of the unmodified sweep to within $2\pp$ at every $\lambda$ (e.g., $33.2\pp$ predicted versus $31.4\pp$ measured at $\lambda=2$), validating the diagnosis on a third teacher whose confidence arises from ImageNet pretraining rather than from a narrow benchmark.

One further property of \eqref{eq:flipcdf} is used in the next section. Serving the teacher at temperature $\tau$ divides every logit gap by $\tau$ while leaving the poison contrast unchanged, so every flip threshold is divided by exactly $\tau$. A threshold of $\lambda^\ast/\tau$ is crossed by a given $\lambda$ precisely when $\lambda^\ast$ is crossed by $\tau\lambda$, so the flipped fraction at strength $\lambda$ under temperature $\tau$ equals the unsoftened fraction at $\tau\lambda$:
\begin{equation}
\label{eq:ftau}
F_\tau(\lambda) \;=\; F(\tau\lambda),
\end{equation}
a parameter-free prediction (Fig.~\ref{fig:tempslide}): temperature shifts the entire transition curve toward smaller $\lambda$ by the factor $1/\tau$. Softening does not alter the mechanism; it rescales the $\lambda$ axis on which the mechanism operates.
\begin{figure}[!t]
  \centering
  \subfloat[CIFAR-100]{\includegraphics[width=\columnwidth]{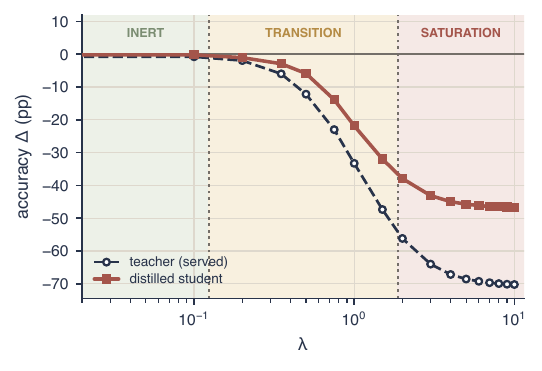}}\\
  \subfloat[CIFAR-10]{\includegraphics[width=\columnwidth]{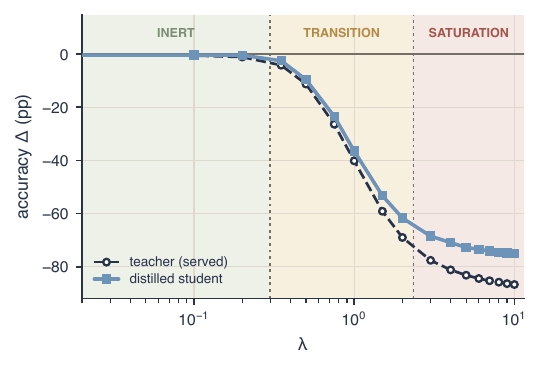}}
  \caption{The inert window and the phase transition (log-$\lambda$). Measured accuracy change of \emph{both} parties under unmodified ADS across the three regimes shaded by the flip-threshold CDF \eqref{eq:flipcdf} (inert: $F\le5\%$; saturation: $F\ge80\%$, near $\lambda\approx2$). In the inert window ($\lambda<0.12$ / $0.30$) the served teacher and the distilled student are both flat within noise---the defense measurably affects neither party---after which both move together through the transition, the teacher falling at least as fast as the student, into saturation.}
  \label{fig:transition}
\end{figure}

The design requirement is now explicit. The penalty direction is sound; what is missing is margin-awareness, and margin-awareness cannot be injected through the poison pipeline. The one free choice inside that pipeline is where the proxy sits, and we tested it: even a proxy distilled directly from the teacher's served probabilities, corresponding to the state a distilling attacker's student would be in during distillation rather than a proxy trained conventionally on labeled data, produces a poison whose contrast--margin correlation remains $\rho \approx 0$ (CIFAR-100). No placement of the proxy makes $\bdelta$ margin-aware, because the objective $\ell$ it derives from never sees a margin. Margin-awareness must therefore be enforced at composition,
the only point where the margin and the poison are simultaneously visible. That is ADS-C (Sec.~\ref{sec:adsc}). First, however, we examine the natural alternative, whose failure the theory above already predicts and whose failure mode is itself informative.

\section{TEMPERATURE SOFTENING}
\label{sec:softening}
Table~\ref{tab:overconfidence} suggests an obvious remedy, that is, if overconfidence starves ADS of non-target mass, soften the served distribution, with a constant distillation temperature~\cite{hinton2015distilling,guo2017calibration} or a confidence-conditioned variant~\cite{balanya2024adaptive,joy2023sample,mozafari2018attended}, to create dark knowledge for the poison to corrupt. This intuition is partially correct, but what it yields is a demonstration of the mechanism rather than a viable defense. Softening relaxes both of the constraints that teacher confidence imposes at once. First, it inflates the non-target mass visible to the attacker's KL loss (activating the poison), and secondly, it shrinks every served margin by the same factor (activating the flips, per \eqref{eq:ftau}). We serve $\sigma(\bz/T)$ for $T\in\{2,4\}$, alone and composed with the raw ADS penalty, over $\lambda\in[0,0.5]$, on both datasets with three seeds, extending the poison-free sweep on CIFAR-10 to $T\in\{6,8,10,12\}$. (Throughout, $\tau$ in \eqref{eq:ftau} and $T$ denote the same served temperature; we write $T$ when quoting experimental settings.) We report four findings.

\begin{figure}[!t]
  \centering
  \subfloat[CIFAR-100]{\includegraphics[width=\columnwidth]{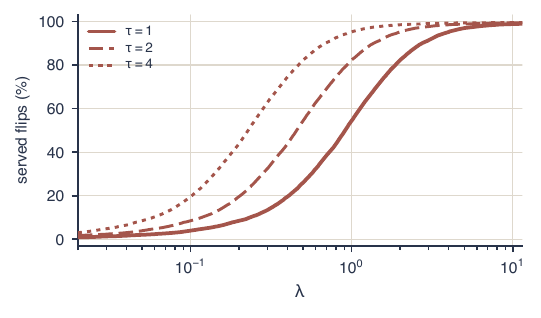}}\\
  \subfloat[CIFAR-10]{\includegraphics[width=\columnwidth]{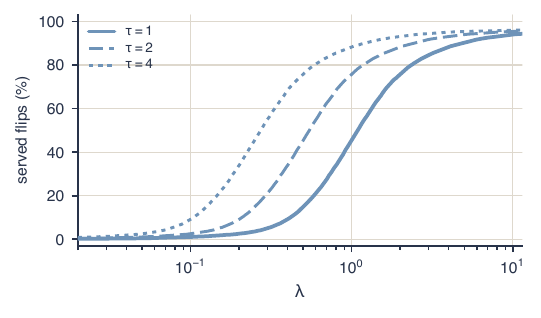}}
  \caption{Temperature slides the transition curve (log-$\lambda$). Served flip fraction under temperature $\tau$ is $F_\tau(\lambda)=F(\tau\lambda)$ \eqref{eq:ftau}: serving at $\tau$ divides every flip threshold by $\tau$, shifting the entire transition curve toward smaller $\lambda$ by the factor $1/\tau$ ($\tau=2,4$ slide $F$ left by $2\times$, $4\times$). This is a parameter-free prediction---no student is trained---and Secs.~\ref{sec:enabler}--\ref{sec:collapse} confirm it against the trained runs.}
  \label{fig:tempslide}
\end{figure}

\subsection{Softening alone can improve the attacker's student}
\label{sec:gift}
Softened teacher outputs are the classic recipe for knowledge distillation~\cite{hinton2015distilling}. However, whether softening assists the adversary is governed by how much dark knowledge the teacher has to reveal. 

On CIFAR-100, with no poison at all ($\lambda=0$), serving at $T=4$ (median served confidence $0.42$) \emph{improves} the distilled student by $+4.7\pp$ over the undefended baseline ($T=2$: $+0.9\pp$). On CIFAR-10 the benefit is absent: pushing to higher temperatures only degrades the student, monotonically, from $-1.0\pp$ at $T=4$ to $-6.5\pp$ at $T=12$. At $T=10$, CIFAR-10 reproduces CIFAR-100's $T=4$ operating point (served median confidence $0.42$ on both), yet CIFAR-100 gains $+4.7\pp$ where CIFAR-10 loses $-5.2\pp$. Tiny-ImageNet interpolates between the two exactly as this account predicts: its non-target mass is intermediate (Table~\ref{tab:overconfidence}), and so is the benefit gained by temperature scaling, that is, $+1.6\pp$ at $T=4$,  which lies between CIFAR-10's absent benefit and CIFAR-100's $+4.7\pp$. At matched served confidence the sign of the effect is therefore set by how much usable non-target structure the teacher has to reveal, not by the magnitude of the temperature parameter or by the class count. Wherever dark knowledge exists, revealing it assists the adversary; where it does not, softening only dilutes the supervision the attacker already had. Note that constant-$T$ scaling preserves the argmax, so this assistance is not even purchased with teacher accuracy: its only cost is served fidelity, which is the axis on which Sec.~\ref{sec:dominated} prices every softening variant.

\subsection{Softening activates ADS inside the inert window}
\label{sec:enabler}
Composed with the poison, softening behaves exactly as \eqref{eq:ftau} predicts: it shifts the entire transition curve into the formerly inert window. At $\lambda=0.2$---inert under plain serving---the poison's marginal effect on the student (measured relative to the same serving at $\lambda=0$, which nets out the shift of Sec.~\ref{sec:gift}) grows from $-1.1\pp$ (unsoftened) to $-3.0\pp$ ($T=2$) to $-17.8\pp$ ($T=4$) on CIFAR-100, and from $-0.4$ to $-5.8$ to $-27.7\pp$ on CIFAR-10 (Fig.~\ref{fig:ect}a): the same small $\lambda$ becomes $17\times$--$69\times$ more potent. This provides direct causal evidence that the availability of non-target mass, not the penalty strength, is what limits the poison at low $\lambda$. CIFAR-10 sharpens the point: the same $T=4$ serving that conferred no benefit on a clean student (Sec.~\ref{sec:gift}) multiplies the poison's potency $69\times$---what softening manufactures is probability mass for the poison to reweight, whether or not that mass encodes anything a student could exploit. The teacher-side cost agrees with the theory quantitatively: the pre-registered prediction $F(\tau\lambda)$ placed the $T=4$, $\lambda=0.2$ teacher cost at $26.9\pp$ (CIFAR-100) and $30.1\pp$ (CIFAR-10); the trained runs measured $25.1$ and $29.1\pp$.

\subsection{The same rescaling collapses the teacher}
\label{sec:collapse}
The teacher-side cost just quoted is the other half of the same mechanism: softening activates the poison by shrinking the very margins the poison then crosses, so enablement and collapse are one phenomenon, not two. On the student-versus-teacher-cost plane (Fig.~\ref{fig:ect}b), every constant-temperature variant ($T=1$, $2$, $4$, on both datasets) collapses onto a single trade-off curve. Temperature never improves the exchange rate rather it only moves the operating point further along it. A worked example makes the severity concrete. A typical softened CIFAR-100 row has served margin $0.38$ against poison contrast $16.8$, so its prediction flips at $\lambda^\ast \approx 0.022$; measured end-to-end, $T=4$ serving loses $25.1$/$29.1\pp$ of teacher accuracy (CIFAR-100/CIFAR-10) at $\lambda=0.2$, a strength that is inert under plain serving, and $45.0$/$56.3\pp$ at $\lambda=0.35$ (Fig.~\ref{fig:ect}). Softening manufactures precisely the thin margins on which a margin-blind poison is most destructive because the two mechanisms compete for the same resource, since the softener consumes the margin the poison needs in order to act safely.

\subsection{At matched served fidelity, softening is dominated}
\label{sec:dominated}
Constant temperature is the bluntest possible softener, so its failure does not by itself close the question. We therefore construct the strongest softening variant we could: a \emph{gated} operator, inspired by~\cite{li2022asymmetric}, that lowers the confidence of over-confident inputs only, and does so by lowering the top logit alone, leaving all non-target structure intact for the dark-knowledge channel (Appendix~\ref{app:scaler_def}). It is composed with the margin budget of Sec.~\ref{sec:adsc}, so that its teacher cost is identically zero. Section~\ref{sec:results_frontier} evaluates this combination across five gates and the full $\lambda$ grid against poison-only ADS-C on the axis a defender actually spends, served fidelity: the softener pays a large fidelity cost \emph{up front}, before any poison flows, and at every matched fidelity level in the operating regime poison-only ADS-C damages the student more (Fig.~\ref{fig:fidelity_frontier}). We therefore include no softening in the proposed defense: softening serves in this paper as the probe that establishes dark knowledge as the resource the poison consumes, and as a defense component it is dominated.

\begin{figure}[!t]
  \centering
  \subfloat[CIFAR-100]{\includegraphics[width=\columnwidth]{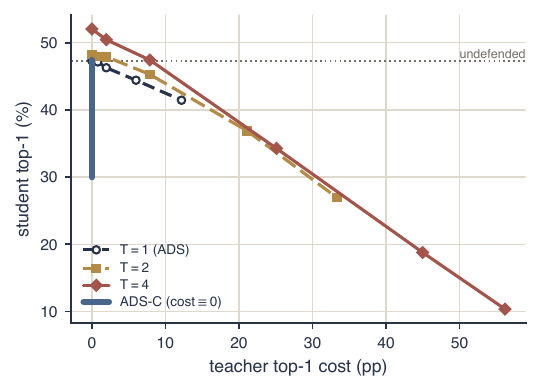}}\\
  \subfloat[CIFAR-10]{\includegraphics[width=\columnwidth]{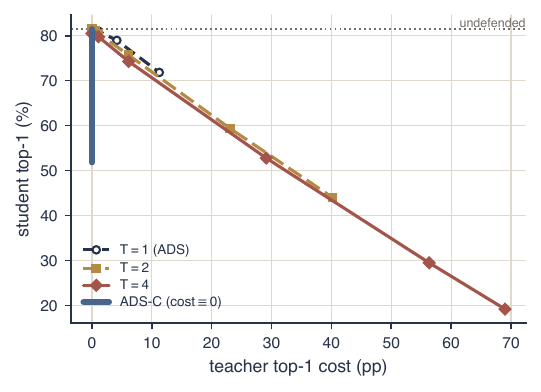}}
  \caption{Constant-temperature serving $\pm$ ADS on the student-versus-teacher-cost trade-off plane (E-CT; 3 seeds). Every constant-temperature arm ($T{=}1,2,4$) collapses onto a \emph{single} unfavorable curve: softening activates the poison---at $\lambda=0.2$ the $T{=}4$ poison is $17\times$/$69\times$ more potent than unsoftened, in quantitative agreement with the pre-registered $F(\tau\lambda)$ prediction---but only by shrinking the margins the poison then crosses, so the extra student damage is bought at proportional teacher cost. ADS-C (vertical line, teacher cost $\equiv0$ by construction) reaches student damage no temperature configuration attains without spending the teacher accuracy the softener pays up front.}
  \label{fig:ect}
\end{figure}

\section{ADS-C}
\label{sec:adsc}
The diagnosis of Sec.~\ref{sec:diagnosis} dictates the design requirement: retain the poison, but enforce margin-awareness at composition, per input. ADS-C implements this requirement in closed form.

\subsection{The margin budget}
Fix an input $\bx$ with teacher logits $\bz$, top class $t=\arg\max_k z_k$, and poison $\bdelta$. For every rival $j\neq t$, let $\mathrm{gap}_j = z_t - z_j$ denote the served margin to that rival and recall from \eqref{eq:contrast} the rate $c_j = \widehat{\Delta}_j - \widehat{\Delta}_t$ at which the poison closes it. Given the requested strength $\lambda$ and a margin floor $m>0$, the per-input budget is
\begin{equation}
\label{eq:budget}
\lambda(\bx) \;=\; \operatorname{clip}\Bigl(\;\min_{j:\,c_j>0}\;\frac{\mathrm{gap}_j - m}{c_j}\;,\;0,\;\lambda\Bigr),
\end{equation}
with the convention that a minimum over an empty set is $+\infty$, and the API serves $\hat{\bp} = \sigma\bigl(\log \bp_\calT + \lambda(\bx)\,\bdelta\bigr)$ (Algorithm~\ref{alg:adsc}). The budget is the flip threshold \eqref{eq:flip} computed with $m$ log-units of every margin held in reserve, clipped to the feasible range, and its behavior at the extremes is the intended one: rivals the poison pushes down ($c_j\le0$) never bind, and an input none of whose rivals gain passes the full requested strength. Conversely, an input whose margin is already at or below the floor receives no poison rather than a flip. Figure~\ref{fig:pipeline} traces one query through the full pipeline.

\begin{algorithm}[!t]
\caption{ADS-C}
\label{alg:adsc}
\begin{algorithmic}[1]
\Require query $\bx$; teacher $\calT$; frozen proxy clones $\calP_{+},\calP_{-}$
\Require step $\varepsilon$; requested strength $\lambda$; margin floor $m$
\Ensure served probability vector $\hat{\bp}$
\State $\bz \gets \calT(\bx)$; \quad $t \gets \arg\max_k z_k$
\State $\bdelta \gets \bigl[\log\sigma(\calP_{+}(\bx)) - \log\sigma(\calP_{-}(\bx))\bigr]/(2\varepsilon)$
\For{$j \neq t$}
  \State $\mathrm{gap}_j \gets z_t - z_j$; \quad $c_j \gets \widehat{\Delta}_j - \widehat{\Delta}_t$
\EndFor
\State $\lambda(\bx) \gets \operatorname{clip}\bigl(\min_{j: c_j>0}(\mathrm{gap}_j - m)/c_j,\; 0,\; \lambda\bigr)$
\State \Return $\hat{\bp} \gets \sigma\bigl(\log\sigma(\bz) + \lambda(\bx)\,\bdelta\bigr)$
\end{algorithmic}
\end{algorithm}

\subsection{The guarantee}
\begin{proposition}[Zero-cost argmax preservation]
\label{prop:argmax}
For every input $\bx$ with unique teacher argmax $t$, the served vector of Algorithm~\ref{alg:adsc} satisfies, for all $j\neq t$,
\[
\log\hat{p}_t - \log\hat{p}_j \;\ge\; \min\bigl(m,\; \mathrm{gap}_j\bigr) \;>\; 0 .
\]
Consequently $\arg\max_k \hat{p}_k = \arg\max_k p_{\calT,k}$ for every query, and the defended teacher's top-1 accuracy equals the undefended teacher's top-1 accuracy identically---as a deterministic, per-query property, not in expectation. Moreover, $\lambda(\bx)$ is maximal: it is the largest value in $[0,\lambda]$ for which the stated margin bound holds.
\end{proposition}
\begin{proof}
Softmax composition preserves log-score differences, so $\log\hat{p}_t - \log\hat{p}_j = \mathrm{gap}_j - \lambda(\bx)\,c_j$. If $c_j \le 0$ this is $\ge \mathrm{gap}_j$. If $c_j > 0$, then by \eqref{eq:budget} $\lambda(\bx) \le (\mathrm{gap}_j - m)/c_j$ whenever that bound is nonnegative, giving $\mathrm{gap}_j - \lambda(\bx)c_j \ge m$; when the bound is negative (i.e., $\mathrm{gap}_j < m$), the clip yields $\lambda(\bx)=0$ and the difference equals $\mathrm{gap}_j$. In all cases the difference is $\ge \min(m, \mathrm{gap}_j) > 0$, so $t$ remains the served argmax. Maximality: constraints with $c_j\le0$ hold for every $\lambda'\ge0$; for $c_j>0$, the constraint $\mathrm{gap}_j - \lambda' c_j \ge \min(m,\mathrm{gap}_j)$ reads $\lambda' \le (\mathrm{gap}_j-m)/c_j$ when $\mathrm{gap}_j\ge m$ and $\lambda'\le0$ when $\mathrm{gap}_j<m$. The feasible subset of $[0,\lambda]$ is therefore the interval $\bigl[0,\;\operatorname{clip}\bigl(\min_{j:c_j>0}(\mathrm{gap}_j-m)/c_j,\,0,\,\lambda\bigr)\bigr]$, whose right endpoint is exactly \eqref{eq:budget}.
\end{proof}

\begin{remark}
The guarantee is unconditional: it requires no retraining, no access to training data, and no distributional assumption, and it holds for every input, including out-of-distribution queries. Empirically we confirm it to machine precision: across all $360$ evaluated (dataset, variant, $\lambda$, seed) configurations, the largest observed teacher top-1 deviation is $1.4\times10^{-6}\pp$---orders of magnitude below the $10^{-2}\pp$ contribution of a single test image, i.e., floating-point noise in the accuracy average rather than any flipped prediction (Sec.~\ref{sec:results}).
\end{remark}

\subsection{Design discussion}
\textbf{The floor $m$.} The floor is the service level the defender guarantees on served margins: every served prediction retains at least $\min(m,\text{original margin})$ log-units of separation. We use $m=0.05$ throughout; the guarantee itself is $m$-independent (teacher drop is zero for every $m$), and Appendix~\ref{app:margin_floor} shows the poison passed through the budget is insensitive over $m\in[0.01,0.2]$ (served poison-KL varies by ${\approx}10\%$ across the range).

\textbf{Overconfidence cuts both ways.} The same heterogeneity that made the global penalty fatal makes the budget cheap: most inputs sit far from the floor, so most of the requested strength survives. At $\lambda=0.5$ the mean effective strength is $0.44$ (CIFAR-100) and $0.47$ (CIFAR-10), corresponding to $88$--$95\%$ of the request, while the thin-margin tail, which a global $\lambda$ would flip first, is individually protected. The budget spends the poison precisely where the confidence structure of Fig.~\ref{fig:overconfidence} has room for it.

\textbf{Deployment properties.} ADS-C is inference-time only. Per query, beyond the teacher's forward pass, it costs two forward passes through the frozen proxy clones, each substantially smaller than the teacher, and $O(K)$ arithmetic for \eqref{eq:budget}. The served output remains a valid probability vector with unchanged argmax. Like ADS, and unlike static perturbation defenses~\cite{lee2019defending,chen2023d}, the perturbation is input-dependent and computed from state the attacker cannot observe (the proxy and its gradient $g$), so there is no fixed transformation to invert. Additionally, unlike ADS, its utility cost is identically zero.

\begin{figure}[!t]
  \centering
  \includegraphics[width=\columnwidth]{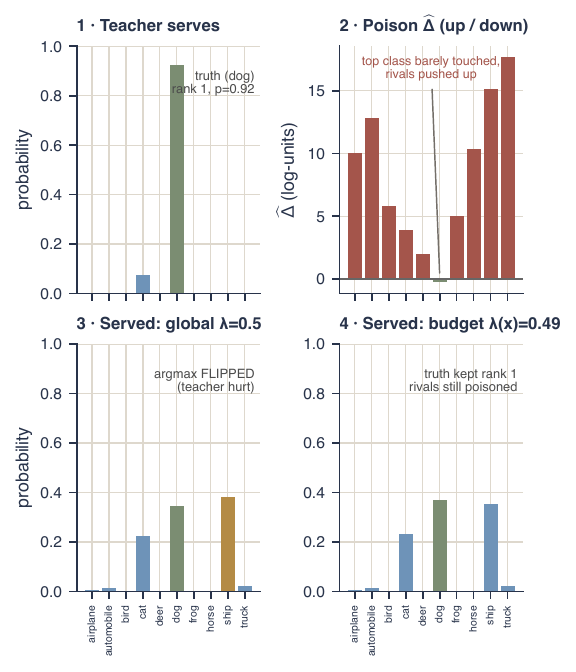}
  \caption{One CIFAR-10 query through the pipeline (panels~1--4). The teacher serves a confident vector (panel~1; truth \emph{dog}, $p_{\max}{=}0.92$); the poison $\bdelta$ leaves the top class essentially untouched ($\widehat{\Delta}_t{=}{-0.3}$) while pushing every rival up, hardest on deep-ranked classes (panel~2); a global $\lambda=0.5$ flips the served argmax to the hard-pushed \emph{ship} (panel~3, teacher hurt), while the margin budget \eqref{eq:budget} spends $\lambda(\bx)=0.49$---just short of the flip---leaving the top-1 intact and the dark knowledge corrupted (panel~4). Pushes surface in the served vector only where the teacher leaves probability mass (\emph{ship}, \emph{cat}); equally large pushes on near-massless classes remain invisible at this strength---the starvation of Sec.~\ref{sec:overconf} in miniature.}
  \label{fig:pipeline}
\end{figure}

\section{EXPERIMENTS}
\label{sec:results}

\subsection{Setup}
\label{sec:setup}
\textbf{Benchmarks.} CIFAR-100 (teacher ResNet-110 $\to$ student ResNet-20), CIFAR-10 (ResNet-56 $\to$ ResNet-20), and Tiny-ImageNet ($K=200$, $64\times64$; teacher ResNet-50, ImageNet-pretrained and fine-tuned with a small-input stem $\to$ student ResNet-18). The proxy always matches the student architecture; Appendix~\ref{app:repro} reports training recipes. Undefended anchors: CIFAR-100 teacher $71.23\%$, distilled student $47.34\pm0.37\%$; CIFAR-10 teacher $92.70\%$, student $81.44\pm0.43\%$; Tiny-ImageNet teacher $80.07\%$, student $47.84\pm0.59\%$.

\textbf{Attacker.} The fixed distillation attacker of Sec.~\ref{sec:threat}: it queries the API once per image with the full unlabeled training split ($50\,000$ queries on the CIFAR datasets, $100\,000$ on Tiny-ImageNet), unaugmented and at native resolution, then trains the student for 50 epochs on the returned (image, served-vector) pairs with the pure soft-label loss \eqref{eq:ukd} at temperature 1 (full optimizer settings in Appendix~\ref{app:repro}). Every configuration is run over three seeds ($\{92, 786, 1337\}$); we report mean\,$\pm$\,std.

\textbf{Defense grid.} $\lambda \in \{0, 0.1, 0.2, 0.35, 0.5, 0.75, 1, 1.5, 2, 3\}$; margin floor $m=0.05$; $\varepsilon = 3\times10^{-3}$ (CIFAR) and $10^{-2}$ (Tiny-ImageNet; Appendix~\ref{app:epsilon}). \emph{ADS} denotes the unmodified transfer \eqref{eq:ads_compose}; \emph{ADS-C} denotes Algorithm~\ref{alg:adsc}; \emph{ADS-C + softener} denotes the ablation of Sec.~\ref{sec:results_frontier}; the guarantee relaxations $q$ and $\lammin$ (Sec.~\ref{sec:dials}) default to $q=1$, $\lammin=0$.

\subsection{Main result: the trade-off rate before and after budgeting}
\label{sec:results_headline}
Figure~\ref{fig:main_tradeoff} plots student accuracy against teacher accuracy for both methods on all three benchmarks; Table~\ref{tab:headline} reports representative operating points.

\textbf{Unmodified ADS.} The direct transfer degrades the student only by degrading the teacher comparably or faster. On CIFAR-100 the sweep passes through $(-12.2\pp\text{ teacher}, -5.9\pp\text{ student})$ at $\lambda=0.5$ and ends at $(-64.0\pp, -43.1\pp)$ at $\lambda=3$. The teacher suffers more than the student. Tiny-ImageNet exhibits the same shape shifted by its wider inert window (Sec.~\ref{sec:transition}). Its exchange is still near-inert at $\lambda=0.5$ ($-1.7\pp$ teacher, $0.0\pp$ student), first becomes active at $\lambda=1$ ($-9.0\pp$ teacher for $-4.6\pp$ student, rate $0.51$), and reaches $(-46.7\pp, -28.8\pp)$ at $\lambda=3$.

\textbf{ADS-C.} With the budget active, teacher top-1 drop is $0.00\pp$ at every cell, while the student falls monotonically with $\lambda$. It falls to $29.95\pm0.23\%$ on CIFAR-100 ($-17.4\pp$, $37\%$ relative), $51.80\pm0.38\%$ on CIFAR-10 ($-29.6\pp$, $36\%$ relative), and $34.51\pm0.28\%$ on Tiny-ImageNet ($-13.3\pp$, $28\%$ relative) at $\lambda=3$.

\textbf{Matched-damage comparison.} For unmodified ADS to inflict the student damage that ADS-C delivers at zero cost, it must spend $27.5\pp$ (CIFAR-100), $32.9\pp$ (CIFAR-10), and $22.2\pp$ (Tiny-ImageNet) of teacher accuracy. The margin budget converts an unfavorable trade into a free one.

\begin{figure*}[!t]
  \centering
  \subfloat[CIFAR-100]{\includegraphics[width=0.32\textwidth]{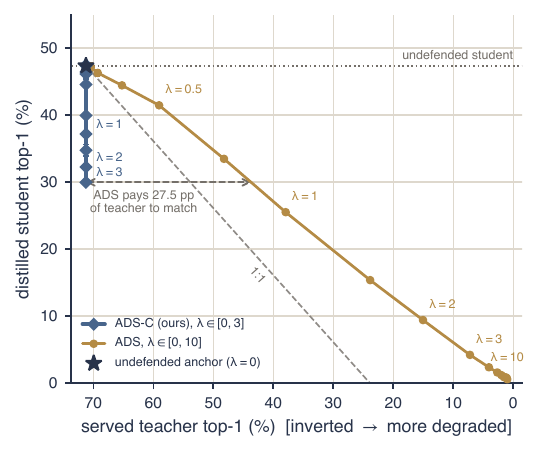}}\hfill
  \subfloat[CIFAR-10]{\includegraphics[width=0.32\textwidth]{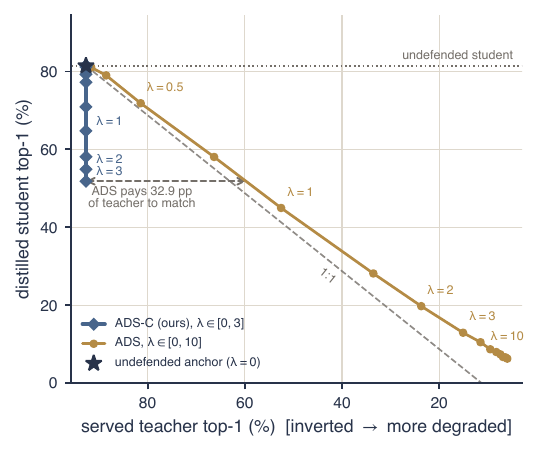}}\hfill
  \subfloat[Tiny-ImageNet]{\includegraphics[width=0.32\textwidth]{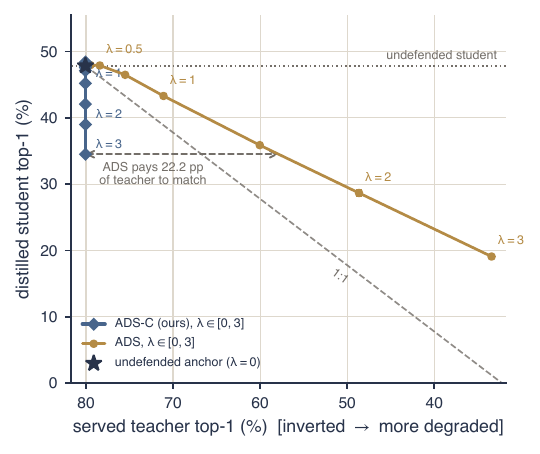}}
  \caption{The main result: distilled-student top-1 versus served-teacher top-1 (horizontal axis inverted---the teacher degrades left to right; 3 seeds; the swept $\lambda$ is annotated along each curve, with ADS extended to $\lambda=10$ on the CIFAR panels; dashed gray: the 1:1 break-even through the undefended anchor $\star$). Unmodified ADS traces a diagonal: it degrades the student only by degrading the teacher comparably or faster. ADS-C is the vertical segment at zero teacher cost (Proposition~\ref{prop:argmax}). The horizontal arrow marks the teacher accuracy unmodified ADS must spend to match ADS-C's free student damage at $\lambda=3$.}
  \label{fig:main_tradeoff}
\end{figure*}

\begin{table}[!t]
\centering
\renewcommand{\arraystretch}{1.05}
\caption{Representative operating points (mean$\pm$std over 3 seeds). Teacher drop for ADS-C is exactly $0$ by Proposition~\ref{prop:argmax} (observed $\le 1.4\times10^{-6}\pp$).}
\label{tab:headline}
\begin{tabular}{llccc}
\toprule
& \textbf{Defense} & $\bm{\lambda}$ & \textbf{$\calS$ top-1 (\%)} & \textbf{$\calT$ drop (pp)} \\
\midrule
\multirow{6}{*}{\rotatebox{90}{CIFAR-100}}
& None    & --- & $47.34\pm0.37$ & $0.00$ \\
& ADS     & $0.5$ & $41.44\pm0.09$ & $12.19$ \\
& ADS     & $3$   & $4.21\pm0.05$  & $64.01$ \\
& ADS-C   & $1$   & $37.16\pm0.27$ & $\mathbf{0.00}$ \\
& ADS-C   & $3$   & $\mathbf{29.95\pm0.23}$ & $\mathbf{0.00}$ \\
& Hard-label & --- & $47.01\pm0.64$ & $0.00$ \\
\midrule
\multirow{6}{*}{\rotatebox{90}{CIFAR-10}}
& None    & --- & $81.44\pm0.43$ & $0.00$ \\
& ADS     & $0.5$ & $71.83\pm0.27$ & $11.26$ \\
& ADS     & $3$   & $12.94\pm0.05$ & $77.61$ \\
& ADS-C   & $1$   & $64.74\pm0.10$ & $\mathbf{0.00}$ \\
& ADS-C   & $3$   & $\mathbf{51.80\pm0.38}$ & $\mathbf{0.00}$ \\
& Hard-label & --- & $81.55\pm0.47$ & $0.00$ \\
\midrule
\multirow{5}{*}{\rotatebox{90}{Tiny-ImageNet}}
& None    & --- & $47.84\pm0.59$ & $0.00$ \\
& ADS     & $1$ & $43.29\pm0.20$ & $8.97$ \\
& ADS     & $3$ & $19.09\pm0.15$ & $46.65$ \\
& ADS-C   & $1$ & $45.19\pm0.41$ & $\mathbf{0.00}$ \\
& ADS-C   & $3$ & $\mathbf{34.51\pm0.28}$ & $\mathbf{0.00}$ \\
\bottomrule
\end{tabular}
\end{table}

\subsection{Decomposition of the unbudgeted poison's damage}
\label{sec:results_decomp}
The budget also functions as an analytical instrument. It separates the unbudgeted poison's student damage into an argmax-safe component, which survives the budget, and flip collateral, which the budget removes (Fig.~\ref{fig:decomposition}). At matched $\lambda=0.5$ on CIFAR-100, the unbudgeted poison's $5.9\pp$ student drop decomposes into $2.8\pp$ of argmax-safe corruption (retained at zero teacher cost) and $3.1\pp$ of flip collateral (purchased at $12.2\pp$ of teacher accuracy); on CIFAR-10, $5.0\pp$ of $9.6\pp$ survives ($52\%$). At smaller strengths the argmax-safe share is larger still ($87\%$ at $\lambda=0.2$ on CIFAR-100). Two conclusions follow. First, roughly half of the student damage that unmodified ADS purchased with teacher accuracy was genuine dark-knowledge corruption that required no payment at all. Second, because the cost ceiling is removed, ADS-C can increase $\lambda$ far past the point where unmodified ADS becomes unusable, more than recovering the remainder: its zero-cost $-17.4\pp$ (CIFAR-100) / $-29.6\pp$ (CIFAR-10) at $\lambda=3$ exceeds the unbudgeted poison's total damage at any teacher cost below $27$--$33\pp$. This decomposition is, to our knowledge, the first quantitative account of \emph{how} ADS corrupts dark knowledge in classification: the durable damage is corruption of non-target structure, not label noise. This conclusion is further corroborated by Sec.~\ref{sec:dials_lmin} from an independent direction.

\begin{figure}[!t]
  \centering{
  \includegraphics[width=\columnwidth]{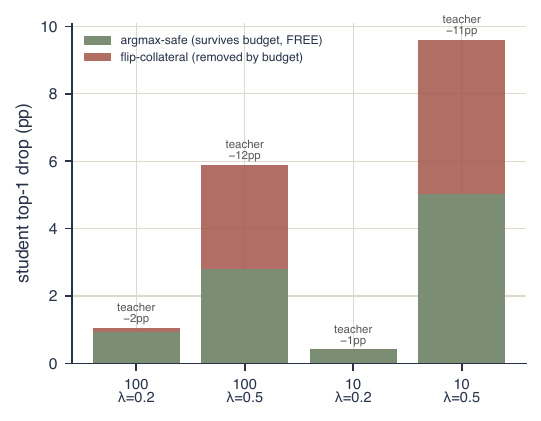}}\\
  \caption{Decomposition of unmodified ADS damage at matched $\lambda$. The unbudgeted student drop splits into an argmax-safe part (survives the budget, at zero cost) and flip collateral (purchased at $11$--$12\pp$ of teacher accuracy).}
  \label{fig:decomposition}
\end{figure}

\subsection{Ablation: softening front-ends at matched served fidelity}
\label{sec:results_frontier}
Finally we place the strongest softening variant, namely the gated top-logit reduction of Sec.~\ref{sec:dominated}, composed with the ADS-C against ADS-C with no temperature scaling. We sweep the softener's gate over five settings and $\lambda$ over the full grid ($360$ runs total across both datasets; teacher drop $\equiv 0$ everywhere, as guaranteed).

At matched $\lambda$ the comparison is misleading as the softener is stronger at low $\lambda$ (by up to $6.3\pp$ on CIFAR-100 and $22.2\pp$ on CIFAR-10) and weaker at high $\lambda$. But $\lambda$ is free. The axes on which a defender actually pays are teacher drop (identically zero for both variants) and served fidelity. Recomputed on the fidelity axis (Fig.~\ref{fig:fidelity_frontier}), the comparison inverts as the softener spends a large fidelity budget up front ($\dkl = 0.49$ at $\lambda=0$ for its most aggressive gate) before any poison flows, while ADS-C spends fidelity only as $\lambda$ grows. At matched served-KL, ADS-C without any softening damages the student more everywhere in the operating regime. Its advantage widens with fidelity spend, from $+1.0\pp$ at KL $0.5$ to $+3.8\pp$ at KL $0.9$ on CIFAR-100, and $+1.0$ to $+2.2\pp$ across KL $0.4$--$0.7$ on CIFAR-10. ADS-C's deepest zero-cost damage ($29.95$/$51.80$) is unmatched by any softener configuration on either dataset. The softener retains a $\le1\pp$ advantage only in the low-damage corner (KL $\le 0.3$) that is undesirable for defense. The matched-$\lambda$ crossover is thus an artifact of comparing on an axis that costs nothing; at matched served fidelity, softening is dominated. This ablation also answers the natural question of why ADS-C includes no confidence-scaling component: the strongest such component we constructed is dominated by not including one.

\begin{figure}[!t]
  \centering
  \subfloat[CIFAR-100]{\includegraphics[width=\columnwidth]{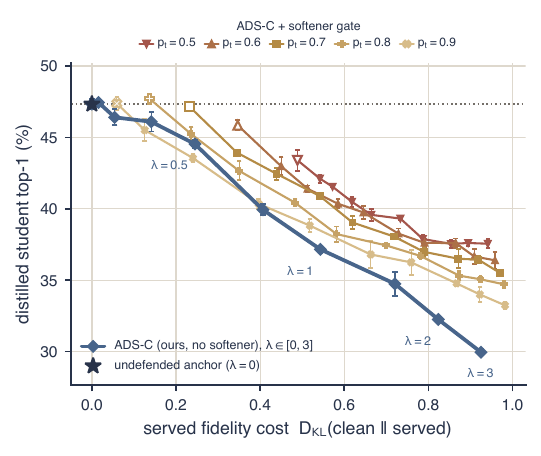}}\\
  \subfloat[CIFAR-10]{\includegraphics[width=\columnwidth]{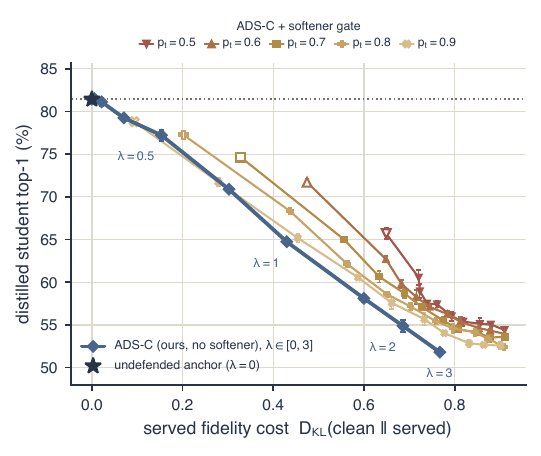}}
  \caption{The served-fidelity trade-off: distilled-student accuracy versus served $\dkl(\bp_\calT\Vert\hat{\bp})$; teacher drop $\equiv0$ for every point (mean$\pm$std whiskers over 3 seeds; $\star$ marks the undefended $\lambda{=}0$ anchor). Each softener gate starts already displaced right (open markers): fidelity spent up front, at $\lambda{=}0$, before any poison flows---up to $\dkl=0.49$ (CIFAR-100) and $0.65$ (CIFAR-10) for the most aggressive gate $p_t{=}0.5$. Poison-only ADS-C is the lower envelope throughout the operating regime, and at the deepest fidelity it spends ($\lambda{=}3$) the best softener configuration at matched served-KL leaves the student $4.0\pp$ / $2.6\pp$ higher (arrows): softening never recovers its up-front spend.}
  \label{fig:fidelity_frontier}
\end{figure}

\subsection{Controls}
\label{sec:results_controls}
\textbf{The hard-label attacker.} We train the attacker of Sec.~\ref{sec:threat} on the served argmax alone (cross-entropy; implemented as KL against the one-hot served label, which is identical). Under ADS-C the served argmax equals the clean argmax on every query. We verify zero deviations over both $50\,000$-query corpora at $\lambda\in\{2,3\}$. Therefore this attacker receives identical supervision from the defended and undefended APIs. Two observations follow. First, these overconfident teachers serve vectors so nearly one-hot that, without softening, dark knowledge confers almost no advantage to begin with, consistent with Table~\ref{tab:overconfidence}, and complementary to Sec.~\ref{sec:gift}, where the advantage appears ($+4.7\pp$) as soon as softening reveals the non-target structure. Second, and decisive for the defense: ADS-C drives the soft-label student accuracy down by $17.1\pp$ (CIFAR-100, $\lambda=3$) and $29.7\pp$ (CIFAR-10). The defended soft output is strictly worse distillation signal than the argmax it accompanies. The defense does not merely remove the incentive to distill from served probabilities, it reverses it, while the argmax channel remains exactly as informative as any accurate API must be. Sec.~\ref{sec:dials} addresses the attacker who exploits knowledge of the guarantee itself.

\subsection{Fidelity of the stolen copy}
\label{sec:results_fidelity}
Top-1 accuracy understates the damage to an attacker whose goal is a faithful replica rather than a merely accurate one (e.g., to mount downstream evasion or membership-inference attacks~\cite{papernot2017practical,shokri2017membership}). Agreement between student and teacher predictions falls in step with student accuracy (CIFAR-100: $0.49$ undefended $\to$ $0.31$ at $\lambda=3$; CIFAR-10: $0.83\to0.52$): the extracted model is not only worse, it is a worse copy.

\section{TRADING THE GUARANTEE FOR ATTACKER UNCERTAINTY}
\label{sec:dials}
Proposition~\ref{prop:argmax} is deliberately absolute, and absoluteness is itself information: an attacker who knows the defense knows that every served top-1 is honest, and that inputs whose margin lies at or below the threshold are served identically to the clean teacher. Those provably clean rows are, moreover, exactly the thin-margin inputs where near-boundary class structure, corresponding to the richest dark knowledge, resides. A defender may therefore wish to spend a small, controlled amount of the guarantee to remove these certainties. We provide two such relaxations, each with closed-form-predictable utility cost and each disabled by default; ADS-C with $q=1$ and $\lammin=0$ remains the proposed method. All results below are trained-student measurements (3 seeds, both CIFAR datasets).

\subsection{Stochastic enforcement of the budget}
\label{sec:dials_q}
The first relaxation enforces the budget \eqref{eq:budget} on each input with probability $q$; with probability $1-q$ the input receives the full requested $\lambda$ (raw poison). The endpoints recover the two extremes---$q=1$ is ADS-C, $q=0$ is unmodified ADS---and the expected teacher cost follows in closed form from Sec.~\ref{sec:transition}: $(1-q)$ times the unmodified-ADS teacher cost at $\lambda$ predicted by \eqref{eq:flipcdf}, so the defender can price the relaxation from cached outputs before deploying it. The enforcement mask is drawn once per input from a fixed seed, not per query (Sec.~\ref{sec:dials_det}).

Figure~\ref{fig:dials_q} shows the resulting family of trade-off curves, which interpolates between the ADS-C vertical and the ADS diagonal. The trade-off is ordered \emph{exactly} by $q$ on both datasets: at every level of teacher cost, the best achievable student damage is obtained by the largest $q$ whose curve reaches that cost, so no configuration with smaller $q$ is ever preferable at equal spend. The exchange below the zero-cost floor is far better than unmodified ADS offers: relaxing to $q=0.9$ at $\lambda=3$ yields a further $-4.1\pp$ of student damage (CIFAR-100, to $25.9\%$) for $6.4\pp$ of teacher cost, where unmodified ADS would require $32.8\pp$ to match---approximately a $5\times$ better exchange; on CIFAR-10, $-6.0\pp$ for $7.9\pp$ against $39.3\pp$ ($5.0\times$). Extending the sweep to $\lambda=10$ on CIFAR-100 strengthens this to a uniform statement: the relaxed configurations saturate ($q=0.9$ reaches student accuracy $22.7\%$ at $6.9\pp$, more than one student-pp per teacher-pp below the floor), and unmodified ADS becomes dominated at \emph{every} point of the trade-off---its deepest operating point (student $4.2\%$ at $64.0\pp$) is surpassed by $q=0.3$ at $\lambda=10$ (student $3.5\%$ at $49.1\pp$).

The hard-label measurements of Sec.~\ref{sec:results_controls} add an important qualification (Fig.~\ref{fig:dials_hl}). Once $q<1$, the two attacker channels separate, and at every $q$ the \emph{hard-label} student is the stronger one ($38.5\%$ versus $25.9\%$ at $q=0.9$ on CIFAR-100; $73.1$ versus $45.8\%$ on CIFAR-10): a defense-aware attacker responds to the relaxation by distilling from the served argmax, whose label noise---the flipped rows---arrives at the closed-form rate $(1-q)F(\lambda)$. Priced against this best response, the exchange below the floor is $1.35$, $1.02$, $0.89$, and $0.81$ student-pp per teacher-pp at $q=0.9$, $0.7$, $0.5$, and $0.3$ on CIFAR-100 ($1.07$, $1.02$, $0.99$, and $0.97$ on CIFAR-10): approximately break-even near the guarantee, drifting below it at deep relaxation, and in every case far from the $5\times$ available against an attacker committed to soft labels. Three observations delimit the relaxation's value. First, the best-responding attacker still degrades below the floor at every $q$---the deliberate flips are systematic, input-consistent noise that the student learns rather than averages away---so no channel recovers the undefended leakage. Second, the exchange rates place the defensible operating points at shallow relaxation ($q\gtrsim0.7$), where the trade stays at or above break-even against either channel; deeper relaxation is justified only against an attacker known to consume soft labels. Third, the guarantee point $q=1$ remains uniquely favorable: it is the only setting at which the defender pays nothing while the attacker's best response yields exactly the floor.

\begin{figure}[!t]
  \centering
  \subfloat[CIFAR-100]{\includegraphics[width=\columnwidth]{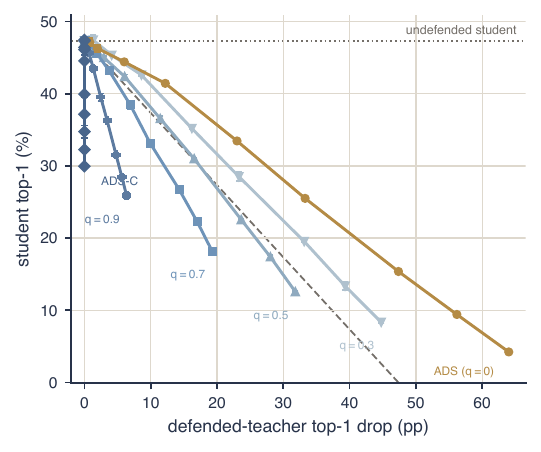}}\\
  \subfloat[CIFAR-10]{\includegraphics[width=\columnwidth]{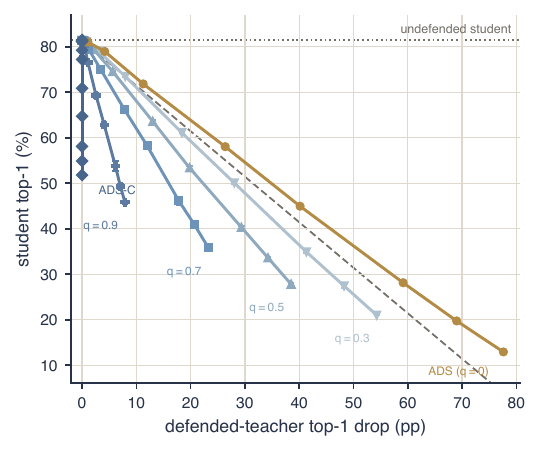}}
  \caption{Stochastic enforcement of the budget (margin cap enforced with probability $q$; 3 seeds; dashed gray: 1:1 break-even). The family of trade-off curves interpolates between ADS-C ($q=1$, vertical) and ADS ($q=0$, diagonal) and is ordered exactly by $q$ on both datasets. Below ADS-C's zero-cost floor, $q=0.9$ obtains additional student damage at approximately $5\times$ better exchange rate than unmodified ADS against the soft-label attacker; Fig.~\ref{fig:dials_hl} prices the same relaxation against the attacker's best response.}
  \label{fig:dials_q}
\end{figure}

\begin{figure}[!t]
  \centering
  \subfloat[CIFAR-100]{\includegraphics[width=\columnwidth]{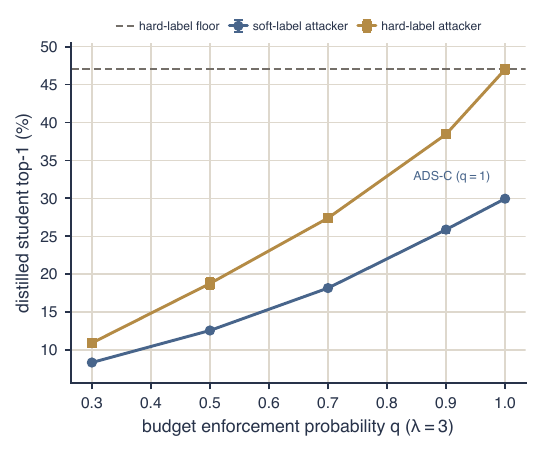}}\\
  \subfloat[CIFAR-10]{\includegraphics[width=\columnwidth]{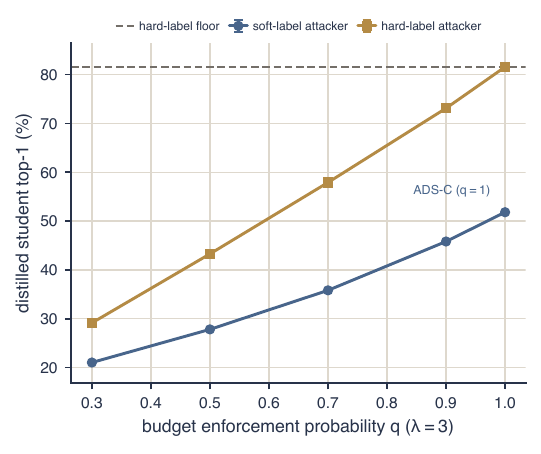}}
  \caption{The two attacker channels under stochastic enforcement ($\lambda=3$, 3 seeds): distilled-student accuracy of the soft-label attacker (served probabilities) and the hard-label attacker (served argmax) as a function of the enforcement probability $q$. At $q=1$ the argmax channel is identical to the undefended API floor (dashed) while the soft channel lies $17.1\pp$ (CIFAR-100) and $29.7\pp$ (CIFAR-10) below it. At every $q<1$ the attacker's best response flips to hard labels, whose deliberate label noise arrives at the closed-form rate $(1-q)F(\lambda)$; both channels nevertheless remain below the floor throughout.}
  \label{fig:dials_hl}
\end{figure}

\subsection{Defensive flipping with a minimum $\lambda$}
\label{sec:dials_lmin}
The second relaxation addresses the clean-row exposure directly: serve every input with strength at least $\lammin$, i.e., $\lambda_{\mathrm{eff}}(\bx)=\max(\lammin, \lambda(\bx))$. Inputs whose flip threshold lies below $\lammin$ are deliberately flipped, which we refer to as \emph{defensive flipping}, concentrated by construction on the thinnest-margin rows, where flips are least costly. The gross cost is again available in closed form from the transition curve, $F(\lammin)$, and the \emph{net} cost is smaller, for a geometric reason: thin-margin rows are disproportionately already misclassified, and on such rows the true class is often the runner-up ($68\%$ of flipped misclassified rows on CIFAR-10, $33\%$ on CIFAR-100), so a fraction of defensive flips land on the true label and partially offset the gross cost (net cost $\approx{\tfrac{1}{2}}$ gross on CIFAR-100; the offset weakens with $K$). This is a consequence of margin geometry rather than of any corrective aim in the poison, and we report it as such.

Measured end-to-end, the minimum-strength floor serves a single purpose, at near-zero cost: at $\lammin=0.05$, $\lambda=3$, \emph{no served row is clean} (the clean-row fraction falls from $0.4\%$ to ${\approx}0$) at a teacher cost of $0.30\pp$ (CIFAR-100) / $0.16\pp$ (CIFAR-10), with student accuracy statistically indistinguishable from the ADS-C floor ($30.1$ vs.\ $29.95$; $51.4$ vs.\ $51.80$). It is a hardening mechanism---it removes the attacker's certainty about which rows are unperturbed---not a means of increasing damage. Nor does it open a side channel: the hard-label attacker under $\lammin\in\{0.05,0.2\}$ remains within seed noise of its leakage floor ($46.7$/$47.3$ vs.\ $47.0\%$ on CIFAR-100; $81.3$/$81.5$ vs.\ $81.6\%$ on CIFAR-10), as the $\le0.2\%$ label-noise rate predicts.

A stronger variant is conceivable: sacrificing the thin-margin rows \emph{entirely} (serving them the full unbudgeted $\lambda$) is the margin-ranked analogue of stochastic enforcement, and on served-distribution corruption (poison-KL) it appeared to dominate stochastic sacrifice by a factor of ${\sim}2$ in our training-free pre-analysis. The trained students overturn this prediction (Fig.~\ref{fig:dials_lmin}): margin-ranked sacrifice cannot push the student meaningfully below the ADS-C floor (best case $29.15\%$ vs.\ floor $29.95\%$ on CIFAR-100, at $6.0\pp$ teacher cost), whereas stochastic sacrifice at the same ${\sim}6\pp$ reaches $25.9\%$ and continues to $8.3\%$ at higher spend. Concentrating heavy poison on near-boundary rows inflates the served KL divergence while teaching the attacker's student little it would not have learned regardless: those rows carry little \emph{informative} dark knowledge relative to their contribution to the KL divergence. We report this negative result because it provides independent, converging evidence for the mechanism identified in Sec.~\ref{sec:results_decomp}: what damages a distilling student is corruption of the informative inter-class structure on representative inputs, not maximization of distributional distortion where distortion is inexpensive.

\begin{figure}[!t]
  \centering
  \subfloat[CIFAR-100]{\includegraphics[width=\columnwidth]{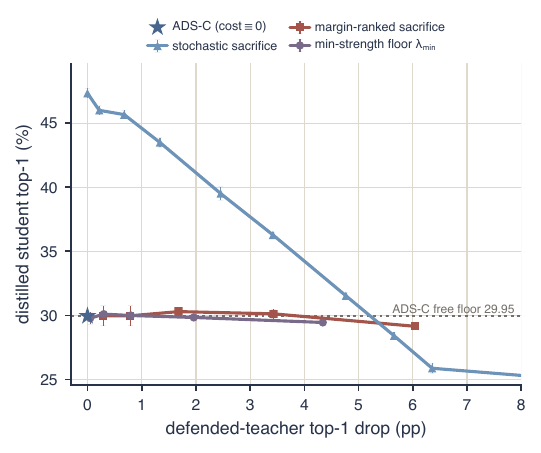}}\\
  \subfloat[CIFAR-10]{\includegraphics[width=\columnwidth]{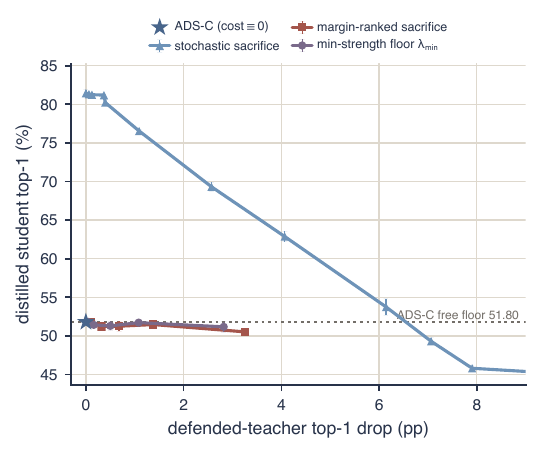}}
  \caption{Defensive flipping versus stochastic sacrifice, in the low-cost region ($\star$: ADS-C at zero teacher cost; dotted rule: the ADS-C free floor; error bars $\pm 1$ std over 3 seeds). Stochastic sacrifice (triangles; the $q$-relaxation of Sec.~\ref{sec:dials_q}, drawn as its lower envelope over the $q,\lambda$ sweep) descends below the floor; margin-ranked sacrifice (squares) is dominated---heavy poison on thin-margin rows inflates served KL divergence but not student damage; and the minimum-strength floor $\lammin$ (circles), like margin-ranked shown at the deployment strength $\lambda{=}3$, sits at the ADS-C floor at $\le0.3\pp$ teacher cost while guaranteeing that no served row is clean.}
  \label{fig:dials_lmin}
\end{figure}

\subsection{Determinism and repeated queries}
\label{sec:dials_det}
Any defense that randomizes \emph{per query} is invertible by repetition: querying the same input many times and aggregating the responses recovers the clean signal. The mechanisms above are immune to this attack by construction. The base ADS-C perturbation is a deterministic function of the input (teacher logits and frozen proxy clones), so repeated queries return the identical perturbed vector; the enforcement mask of Sec.~\ref{sec:dials_q} is drawn per \emph{input} from a fixed seed, not per query, so a sacrificed input is consistently sacrificed; and $\lammin$ is inherently per-input deterministic. Repeated queries therefore provide the attacker no additional information. The residual cost of determinism is a persistent error set under $\lammin$ (the same flipped inputs are always flipped)---for a deployed API arguably preferable to nondeterministic answers---and identifying that set would require access to the clean teacher, which is the asset under protection. The utility statement weakens in form but remains strong: from ``accuracy identical'' to ``accuracy drop at most $F(\lammin)$, known in closed form before deployment, and concentrated on inputs the teacher was already uncertain about.''

\section{DISCUSSION AND LIMITATIONS}
\label{sec:discussion}
\textbf{What the guarantee does and does not cover.} Proposition~\ref{prop:argmax} protects the served argmax---the quantity consumed by classification API users. Applications that consume the full served distribution (e.g., downstream risk calibration) pay the served-fidelity cost, which we quantify explicitly: it is the horizontal axis of Fig.~\ref{fig:fidelity_frontier}, and at the recommended operating point amounts to a served KL of $0.8$--$0.9$ (CIFAR-100). In the dark-knowledge threat model this cost falls almost entirely on the attacker: the honest consumer of a classification API reads the label.

\textbf{Attacker adaptivity.} Our main results use the standard fixed distillation attacker. The hard-label attacker is reduced to the distillation floor every classification API shares. Repetition attacks are addressed by per-input determinism (Sec.~\ref{sec:dials_det}). Attacks that estimate and subtract the poison would need to reconstruct a hidden gradient direction through a black-box interface; we view a formal analysis of this inversion problem as valuable future work.

\textbf{Scale and generality.} Our evidence is CIFAR-10, CIFAR-100, and Tiny-ImageNet. Nothing in ADS-C is architecture- or scale-specific, and the Tiny-ImageNet teacher already covers the deployment-realistic case of a strong pretrained model, whose overconfidence (Table~\ref{tab:overconfidence}) shows the phenomenon we exploit is a property of model quality rather than of small benchmarks. Validation on large-scale production teachers and modern architectures (e.g., vision transformers) remains future work; the transition-curve analysis of Sec.~\ref{sec:transition} provides a tool, computable from teacher and proxy outputs alone, to predict, before any training, how a new teacher will behave.

\textbf{The trade-off-rate lens.} We encourage future inference-time defenses to report (i) utility cost with the axes made explicit, and (ii) frontiers at matched served fidelity rather than matched internal knob settings; Sec.~\ref{sec:results_frontier} shows the latter comparison can invert the former.

\section{CONCLUSION}
\label{sec:conclusion}
We adapted Antidistillation Sampling to classification and showed that its behavior there is governed by teacher overconfidence: an inert window in which the defense measurably affects neither party, a transition curve that predicts the defended teacher's cost across the entire sweep, and a saturating regime in which the direct transfer trades below break-even. The perturbation itself is faithful and correctly aimed at dark knowledge; the failure is margin-blindness at composition, on teachers whose confidence is extremely heterogeneous. Temperature softening rescales the transition curve exactly as the theory predicts, which makes it a valuable probe and demonstrates that the availability of dark knowledge is what limits the perturbation at low strengths; but it is not a remedy: the same rescaling collapses the teacher, and every temperature configuration lies on the same unfavorable trade-off curve. ADS-C enforces the missing margin-awareness where both the margin and the perturbation are visible---at composition---with a closed-form per-input budget that provably preserves every served prediction. The result is, to our knowledge, the first antidistillation defense for classification with exactly zero utility cost: $17.4\pp$ (CIFAR-100), $29.6\pp$ (CIFAR-10), and $13.3\pp$ (Tiny-ImageNet) of distilled-student degradation at a teacher cost of $0.00\pp$. Measured against the complementary channel, the guarantee reverses the economics of extraction: an attacker restricted to the served argmax obtains exactly the distillation floor that every accurate classification API shares, while the defended soft output trains a student well below that floor. The served probabilities become strictly worse supervision than the labels they accompany. Around the guarantee we provide two relaxations with closed-form-predictable cost, stochastic enforcement and defensive flipping, that allow a defender to exchange a known amount of utility for attacker uncertainty, priced throughout against the attacker's best response across both supervision channels, together with a negative result showing that stochastic sacrifice outperforms margin-ranked sacrifice, further evidence that the durable damage resides in informative dark knowledge. A defense that costs nothing changes the deployment calculus as there is no longer a utility argument for serving undefended soft labels.

\section*{ACKNOWLEDGMENTS}
The authors thank the MS Artificial Intelligence Program at the Rochester Institute of Technology for computational resources, and the Fulbright Program for scholarship support to the first author. M. J. Khojasteh acknowledges support from the Gleason Endowment and Provost’s Learning Innovation Grant at RIT. 
The authors acknowledge Research Computing at the Rochester Institute of Technology for providing computational resources and support that have contributed to the research results reported in this publication~\cite{https://doi.org/10.34788/0s3g-qd15}.

\bibliographystyle{IEEEtran}
\bibliography{references}

\begin{IEEEbiography}[{\includegraphics[width=1in,height=1.25in,clip,keepaspectratio]{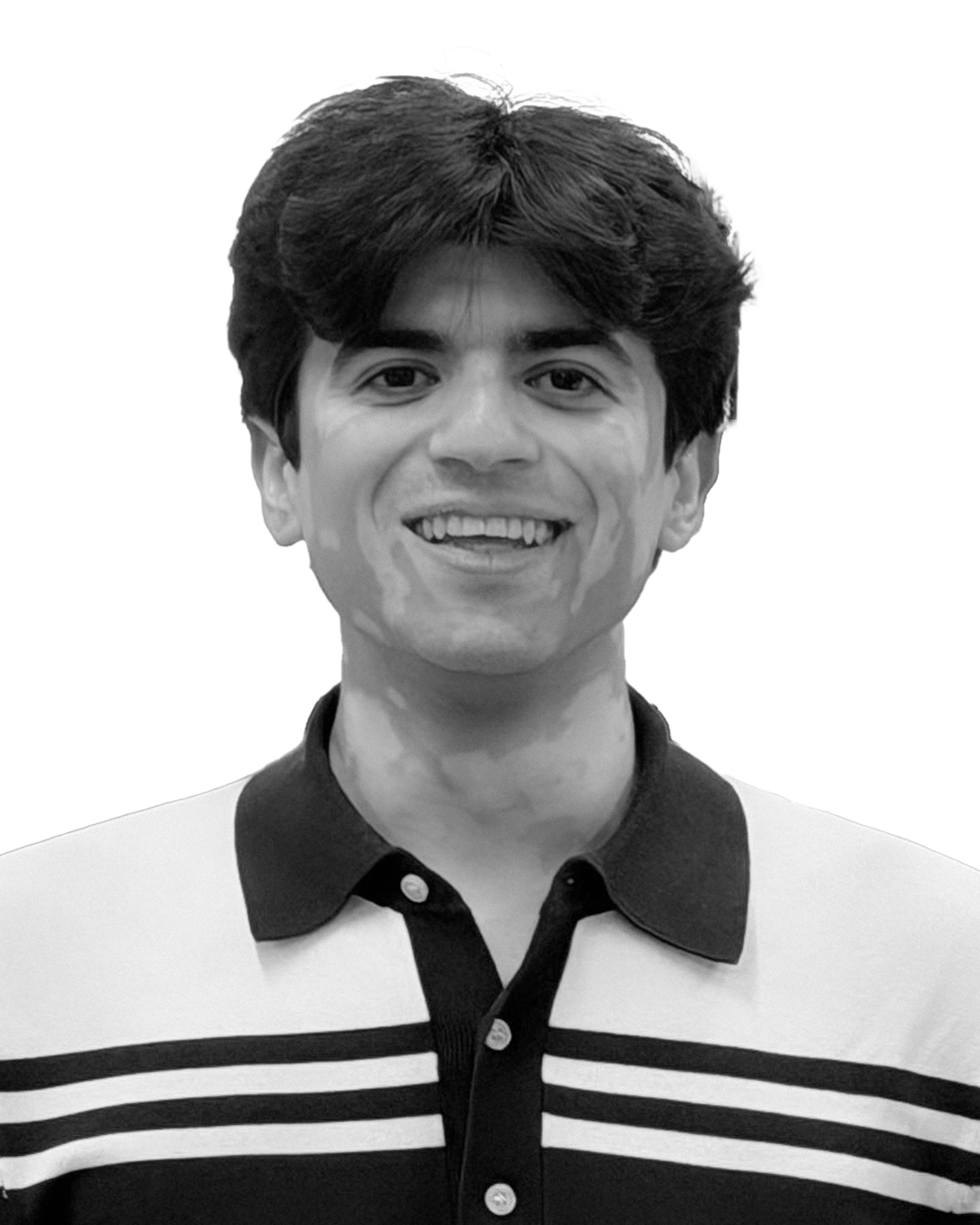}}]
{Khawaja Abaid Ullah}~received the B.S. degree in Computer Science from the University of Narowal, Pakistan in 2021, and the M.S. degree in Artificial Intelligence from the Rochester Institute of Technology (RIT), Rochester, NY, USA, in 2026, as a Fulbright Scholar. He is currently pursuing the Ph.D. degree in Electrical and Computer Engineering at RIT, serving as a Graduate Research Assistant in the Khojasteh Autonomous Systems (KAS) Lab. His research interests include deep learning, robotics, and autonomous systems.

\end{IEEEbiography}

\begin{IEEEbiography}[{\includegraphics[width=1in,height=1.25in,clip,keepaspectratio]{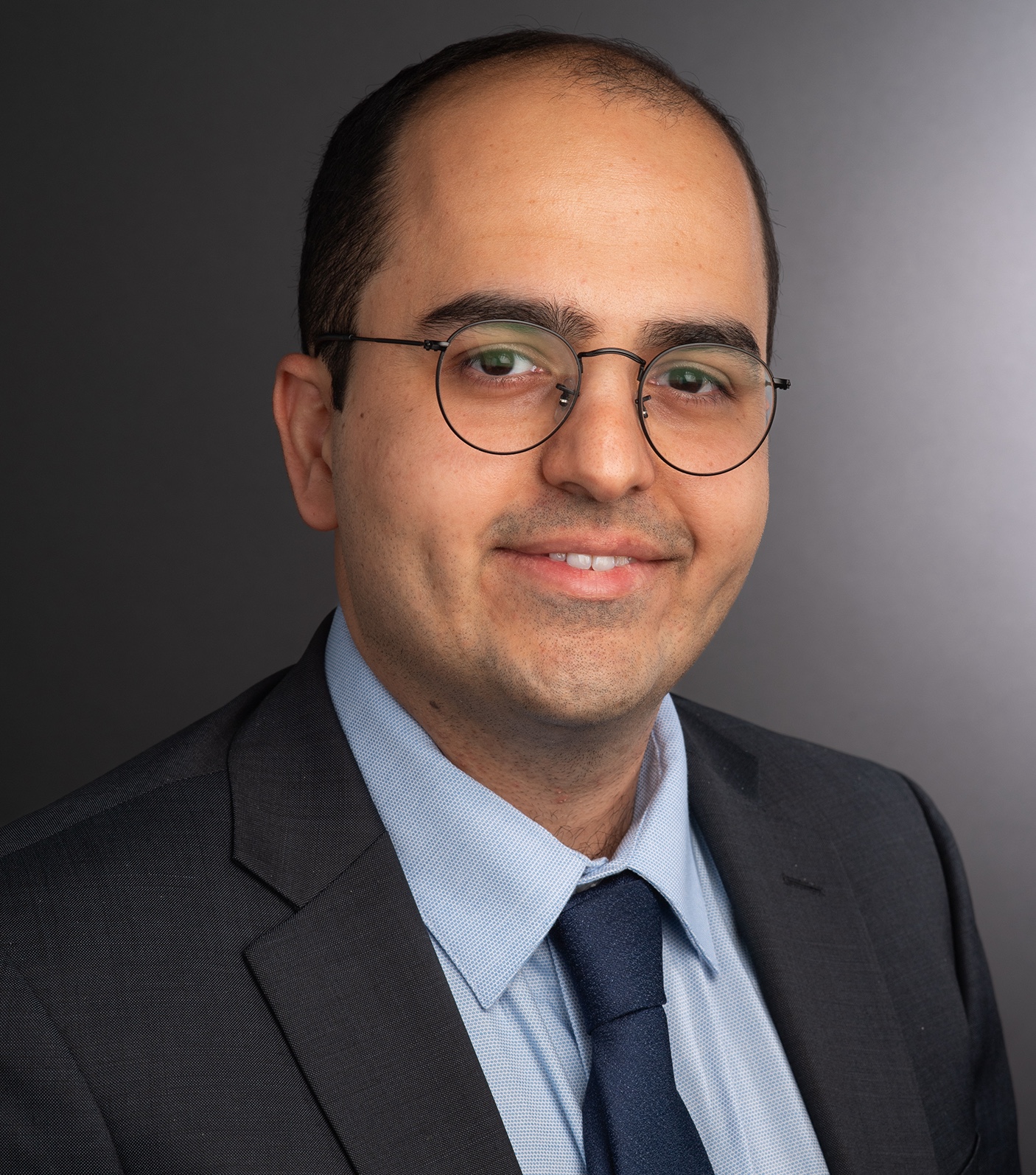}}]{Mohammad Javad Khojasteh} ~(Member,
IEEE) received the Ph.D. and M.Sc. degrees in electrical and computer engineering from the University of California, San Diego in 2019 and 2017, respectively. 
He is a Gleason Endowed Assistant Professor with the Rochester Institute of Technology (RIT). Before joining RIT, he held postdoctoral positions at Marine Physical Laboratory (MPL) at Scripps Institution of Oceanography (SIO), Department of Mechanical Engineering and Laboratory for Information and Decision Systems (LIDS) at Massachusetts Institute of Technology (MIT), and Center for Autonomous Systems and Technologies (CAST) at California Institute of Technology (Caltech), where he worked with Team CoSTAR as visitor at NASA Jet Propulsion Laboratory. He received the Gleason Chair in 2024 and the 2025 Provost’s Learning Innovation Grant at RIT. Dr. Khojasteh's publications, co-authored with colleagues and students, have received awards, including Tammy L. Blair Student Paper Award (second place) from the International Society of Information Fusion. He is an Associate Editor for IEEE Signal Processing Letters.
\end{IEEEbiography}

\section*{APPENDIX}
\subsection{Notation}
\label{app:notation}
Table~\ref{tab:notation} summarizes the symbols used throughout the paper, grouped by role.

\begin{table}[!ht]
\centering
\caption{Summary of notation.}
\label{tab:notation}
\small
\renewcommand{\arraystretch}{1.15}
\setlength{\tabcolsep}{4pt}
\begin{tabular}{@{}>{\raggedright\arraybackslash}p{0.22\linewidth} >{\raggedright\arraybackslash}p{0.72\linewidth}@{}}
\toprule
\textbf{Symbol} & \textbf{Description} \\
\midrule
$K$, $\Delta^{K-1}$ & Number of classes; probability simplex \\
$\calT,\calS,\calP$ & Teacher, student (attacker), proxy models \\
$\calP_{+},\calP_{-}$ & Frozen proxy clones at $\theta_\calP\pm\varepsilon g$ \\
$\bz$, $\sigma(\cdot)$ & Logit vector; softmax \\
$t$, $p_{\max}$ & Served top class; top-class probability \\
$\mathrm{gap}_j$ & Margin $z_t - z_j$ to rival $j$ \\
$\ell$, $g$ & Proxy holdout loss (per-sample mean); its gradient \\
$\widehat{\Delta}(y\mid\bx)$ & Antidistillation penalty \eqref{eq:fd} \\
$c_j$, $c(\bx)$ & Rival gain rate $\widehat{\Delta}_j-\widehat{\Delta}_t$; contrast \eqref{eq:contrast} \\
$\lambda$, $\lambda(\bx)$ & Requested strength; per-input budget \eqref{eq:budget} \\
$\lambda^\ast(\bx)$ & Flip threshold \eqref{eq:flip} \\
$F(\lambda)$, $\lambda_c$ & Flip-threshold CDF \eqref{eq:flipcdf}; transition threshold \\
$m$ & Served-margin floor \\
$\tau$, $T$ & Served temperature (Secs.~\ref{sec:transition}--\ref{sec:softening}) \\
$q$ & Budget-enforcement probability (Sec.~\ref{sec:dials_q}) \\
$\lammin$ & Minimum served strength (Sec.~\ref{sec:dials_lmin}) \\
$\varepsilon$ & Finite-difference step \\
$\hat{\bp}$ & Served (defended) probability vector \\
$\dkl(\cdot\Vert\cdot)$ & Kullback--Leibler divergence \\
\bottomrule
\end{tabular}
\end{table}

\subsection{Calibration of the finite-difference step $\varepsilon$}
\label{app:epsilon}
Following the ADS paper~\cite{savani2025antidistillation}, the penalty \eqref{eq:fd} approximates the exact directional derivative $\langle \nabla_{\theta}\log p(y\mid\bx;\theta_\calP),\, g\rangle$, which we compute directly as a Jacobian--vector product (JVP) on a reference batch. We select $\varepsilon$ by sweeping over a log-grid and scoring the per-coordinate magnitude ratio and angular agreement between \eqref{eq:fd} and the JVP. Too-small $\varepsilon$ loses the signal to floating-point cancellation; too-large $\varepsilon$ leaves the linear regime (Fig.~\ref{fig:epsilon_calibration}). The selected steps, $\varepsilon = 3\times10^{-3}$ on both CIFAR datasets and $\varepsilon = 10^{-2}$ on Tiny-ImageNet, each achieve a magnitude ratio of $\approx1.00$ under our GPU (TF32) inference stack; the calibration is hardware-arithmetic-specific and should be re-run if the serving stack changes numerics.

\begin{figure}[!t]
  \centering
  \includegraphics[width=\columnwidth]{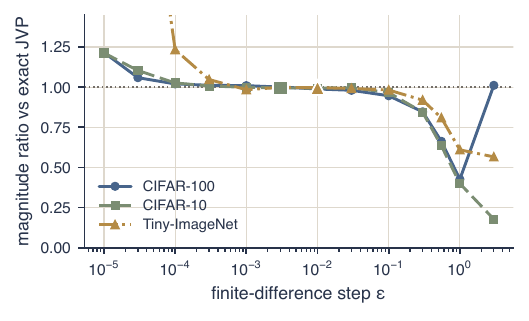}
  \caption{Calibration of the finite-difference step. Magnitude ratio between \eqref{eq:fd} and the exact JVP as a function of $\varepsilon$ (log axis); the enlarged marker on each curve is the selected step. Below the plateau, floating-point cancellation destroys the signal; above it, the finite difference leaves the linear regime.}
  \label{fig:epsilon_calibration}
\end{figure}

\subsection{Sensitivity to the margin floor $m$}
\label{app:margin_floor}
The guarantee is $m$-independent: recomputing the full serving pipeline over $m\in\{0.01,0.05,0.1,0.2\}$ at $\lambda\in\{1,2\}$, the defended teacher's top-1 drop is exactly $0.00\pp$ at every setting on both datasets, as Proposition~\ref{prop:argmax} requires. What $m$ modulates is only how much poison the budget passes, and mildly: at $\lambda=2$ on CIFAR-100 the mean effective strength moves from $1.028$ ($m=0.01$) to $1.009$ ($m=0.2$) and the served poison-KL from $0.842$ to $0.757$ (${\approx}10\%$ across a $20\times$ range of $m$); CIFAR-10 behaves identically ($0.702\to0.627$). At $\lambda=1$ the spread is smaller still. Conclusions are therefore insensitive to $m$; we fix $m=0.05$ as a mid-range served-margin service level. (Rows whose original margin is below $m$ receive zero poison and keep their original margin, per the $\min(m,\text{original})$ guarantee---hence the minimum achieved served margin can lie below $m$.) A student-side confirmation (distilled students at $m\in\{0.01,0.05,0.1,0.2\}$, $\lambda\in\{1,2\}$, three seeds per cell) shows the same insensitivity where it matters (Fig.~\ref{fig:msens_students}): student top-1 varies by at most $1.7\pp$ across the $20\times$ range of $m$ on CIFAR-100 ($\lambda=2$: $32.3\%$ at $m=0.01$ versus $33.9\%$ at $m=0.2$) and by at most $2.9\pp$ on CIFAR-10 ($53.7\%$ versus $56.6\%$), against a poison-induced degradation of $10$--$27\pp$ at these $\lambda$. Larger floors admit marginally less poison, as expected, and the defended teacher's top-1 drop remains exactly zero at every cell.

\begin{figure}[!t]
  \centering
  \subfloat[CIFAR-100]{\includegraphics[width=\columnwidth]{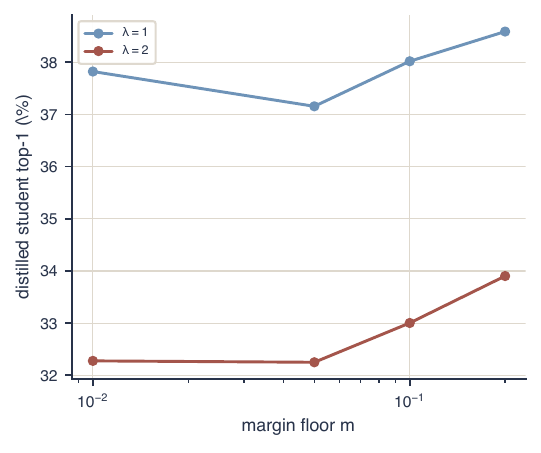}}\\
  \subfloat[CIFAR-10]{\includegraphics[width=\columnwidth]{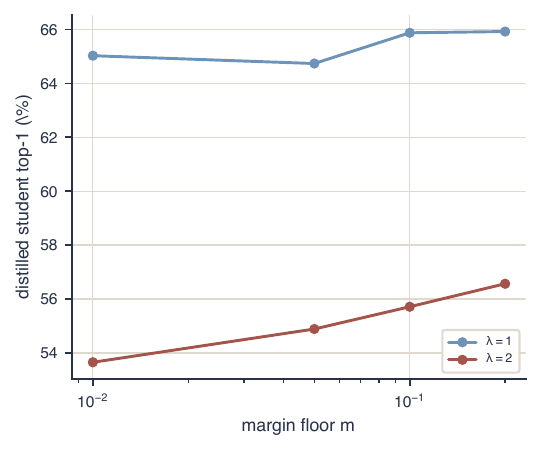}}
  \caption{Student-side sensitivity to the margin floor $m$ ($\lambda\in\{1,2\}$, 3 seeds). Distilled-student top-1 varies by at most $1.7\pp$ (CIFAR-100) and $2.9\pp$ (CIFAR-10) over the $20\times$ range $m\in[0.01,0.2]$, against a poison-induced degradation of $10$--$27\pp$ at these strengths; the defended teacher's top-1 drop is exactly zero at every cell.}
  \label{fig:msens_students}
\end{figure}

\subsection{Softening operators used in the ablation}
\label{app:scaler_def}
The ablation of Sec.~\ref{sec:results_frontier} uses a gated, argmax-preserving softener defined in log-odds space. Let $G = z_t - \operatorname{logsumexp}_{j\neq t} z_j$, which equals $\operatorname{logit}(p_{\max})$ exactly. Given a gate $p_g$ with $G_g = \operatorname{logit}(p_g)$, inputs with $G \le G_g$ are untouched; over-confident inputs are mapped to the served confidence $G' = G_g$ by lowering the top logit only ($z_t \leftarrow G' + \operatorname{logsumexp}_{j\neq t} z_j$), leaving all non-target structure intact for the dark-knowledge channel. We sweep $p_g \in \{0.5, 0.6, 0.7, 0.8, 0.9\}$; the constant-temperature comparison is Sec.~\ref{sec:softening} (E-CT). The operator is inspired in spirit by asymmetric temperature scaling~\cite{li2022asymmetric}, which applies separate temperatures to the correct and wrong classes so that wrong-class discriminability survives softening; ours is the gated, serving-side limiting case in which the wrong-class logits are left exactly unchanged. Additionally, where asymmetric temperature scaling selects a fixed pair of temperatures by grid search over a candidate space, our operator adapts the effective softening per query.

\subsection{Reproducibility}
\label{app:repro}
Table~\ref{tab:recipes} reports the full training recipes. All models are trained with SGD (momentum $0.9$, Nesterov, weight decay $10^{-4}$) under cosine annealing; teachers and proxies use a 5-epoch linear warmup, plain cross-entropy, and \emph{no} label smoothing. Teachers and proxies train on the augmented training split (random crop with 4-pixel padding at $32\times32$, 8-pixel at $64\times64$, plus horizontal flip); the Tiny-ImageNet teacher is an ImageNet-pretrained ResNet-50 fine-tuned end to end with the stem adapted to $64\times64$ (first-conv stride $2\to1$, max-pool removed, fresh 200-way head). The clone gradient $g$ is computed on a held-out $10\%$ of the training split. The attacker's student trains on the cached (image, served-vector) pairs with no augmentation. Every figure is generated by a script that emits a sibling metrics JSON with the underlying numbers, and aggregates for every (variant, $\lambda$, seed) cell are archived. Code, configuration files, and exact command lines: \href{https://github.com/the-kas-lab/ADS-C}{https://github.com/the-kas-lab/ADS-C}.

\begin{table}[!t]
\centering
\renewcommand{\arraystretch}{1.1}
\caption{Training recipes (SGD, momentum $0.9$, Nesterov, weight decay $10^{-4}$, cosine annealing throughout).}
\label{tab:recipes}
\small
\setlength{\tabcolsep}{4pt}
\begin{tabular}{llccc}
\toprule
& \textbf{Model} & \textbf{Epochs} & \textbf{LR} & \textbf{Batch} \\
\midrule
\multirow{3}{*}{\rotatebox{90}{C-100}}
& Teacher (ResNet-110) & 200 & 0.1 & 256 \\
& Proxy (ResNet-20) & 50 & 0.1 & 256 \\
& Student (ResNet-20) & 50 & 0.1 & 512 \\
\midrule
\multirow{3}{*}{\rotatebox{90}{C-10}}
& Teacher (ResNet-56) & 200 & 0.1 & 256 \\
& Proxy (ResNet-20) & 50 & 0.1 & 256 \\
& Student (ResNet-20) & 50 & 0.1 & 512 \\
\midrule
\multirow{3}{*}{\rotatebox{90}{Tiny-IN}}
& Teacher (ResNet-50$^\dagger$) & 40 & 0.02 & 128 \\
& Proxy (ResNet-18) & 60 & 0.1 & 128 \\
& Student (ResNet-18) & 50 & 0.1 & 256 \\
\bottomrule
\multicolumn{5}{l}{\footnotesize $^\dagger$ImageNet-pretrained, fine-tuned end to end at $64\times64$.}
\end{tabular}
\end{table}

\end{document}